\documentclass{article}
\usepackage{comment}
  \usepackage[preprint]{neurips_2026}

\usepackage[utf8]{inputenc} 
\usepackage[T1]{fontenc}    
\usepackage{hyperref}       
\usepackage{url}            
\usepackage{booktabs}       
\usepackage{amsfonts}       
\usepackage{nicefrac}       
\usepackage{microtype}      
\usepackage{xcolor}         

\usepackage{amsthm}

\newtheorem{theorem}{Theorem}[section]

\newtheorem{lemma}[theorem]{Lemma}
\newtheorem{corollary}[theorem]{Corollary}
\newtheorem{remark}[theorem]{Remark}

\newtheorem{proposition}[theorem]{Proposition}

\title{Adversary-Robust Learning from Fully Asynchronous Directional Derivative Estimates}

\usepackage{times}
\usepackage{float}
\usepackage{algorithm}
\usepackage{algorithmic} 
\usepackage{makecell}
\usepackage{multirow} 
\usepackage{xspace}
\usepackage{pifont}
\usepackage{xurl} 
\usepackage{enumitem}
\usepackage{threeparttable}
\usepackage{hyperref} 
\usepackage{tabularx}
\usepackage{amssymb}
\usepackage{amsmath}

\usepackage{diagbox}

\usepackage{booktabs}
\usepackage{pifont}
\newcommand{\cmark}{\ding{51}}
\newcommand{\xmark}{\ding{55}}
\newcommand{\pmark}
{\checkmark^{\dagger}}

\usepackage{graphicx}
\usepackage{float}
\usepackage{placeins}
\usepackage[skip=3pt]{caption}
\usepackage{graphicx}
\usepackage{bm}
\usepackage{booktabs}
\newtheorem{assumption}{Assumption}
\newcommand{\ncom}{\newcommand}

\newcommand{\beqn}{\begin{eqnarray*}}
\newcommand{\eeqn}{\end{eqnarray*}}
\newcommand{\beq}{\begin{eqnarray}}
\newcommand{\eeq}{\end{eqnarray}}

\newcommand{\norm}[1]{\left\lVert #1 \right\rVert}
\newcommand{\inprod}[2]{\left\langle #1, #2 \right\rangle}
\ncom\R{\mathbb{R}}

\DeclareMathOperator*{\sign}{sign}

\newcommand{\rsign}{R-SIGN\xspace}
\newcommand{\farsign}{FAR-SIGN\xspace}
\newcommand{\Sign}{\rm Sign}

\newcommand{\bI}{\mathbb{I}}

\newcommand{\as}{\textnormal{a.s.}}

\newcommand{\bR}{\mathbb{R}}

\newcommand{\cA}{\mathcal{A}}
\newcommand{\cF}{\mathcal{F}}

\newcommand{\cI}{\mathcal{I}}
\newcommand{\cW}{\mathcal{W}}

\author{
\textbf{Anik Kumar Paul}$^{1}$, \textbf{Nibedita Roy}$^{1}$, \textbf{Nagesh Talagani}$^{1}$, \\ \textbf{Swetha Ganesh}$^{2}$, \textbf{Gugan Thoppe}$^{1}$, \textbf{Alexandre Reiffers-Masson}$^{3}$\\
$^{1}$Computer Science and Automation, Indian Institute of Science, Bengaluru\\
$^{2}$ Edwardson School of Industrial Engineering, Purdue University\\
$^{3}$Department of Computer Science, IMT Atlantique\\
\vspace{1mm}
\texttt{anikpaul42@gmail.com, nibeditaroy@iisc.ac.in,} \\ \texttt{nagesht@iisc.ac.in,
swethaganesh@iisc.ac.in,}\\ \texttt{gthoppe@iisc.ac.in, alexandre.reiffers-masson@imt-atlantique.fr}
}

\begin{document}

\maketitle

\begin{abstract}
    We propose \farsign{} (Fully Asynchronous Robust optimization via SIGNed directional projections) for adversary-resilient learning in parameter-server--worker systems. \farsign{} achieves robustness through sign-based updates along carefully designed directions and mitigates the resulting bias via a two-timescale mechanism. It admits both first-order and zeroth-order implementations and enables fully asynchronous execution without requiring a private reference dataset at the server. We establish almost-sure convergence of \farsign{} to the set of stationary points for smooth, nonconvex objectives. Moreover, we prove the near-optimal rate of $O(n^{-1/4+\epsilon})$ in the first-order setting and the standard $O(n^{-1/6+\epsilon})$ in the zeroth-order setting, where $n$ is the iteration count and $\epsilon>0$ can be chosen arbitrarily small. Experiments on MNIST show that \farsign{} outperforms robust aggregation-based methods in both accuracy and wall-clock time.
\end{abstract}

\section{Introduction}

Modern machine learning systems increasingly rely on distributed data and worker-side computation to solve large-scale optimization problems \citep{bottou2018optimization}. 
By leveraging multiple worker machines to process data and provide gradient or function-value feedback, such systems enable scalability well beyond what a single processor can achieve \citep{konevcny2016federated,McMahan2017federated}. 
This distributed computational model, however, introduces significant challenges in robustness and reliability. 
Worker nodes may fail due to hardware faults, stalled computations, or unreliable communication. 
Alternatively, they may behave adversarially, deliberately transmitting corrupted information to degrade system performance \citep{lamport1982byzantine}. 
Addressing such failures and adversarial behavior is therefore pivotal for reliable large-scale learning.

A large literature has studied robust learning in parameter-server--worker architectures with adversarial or faulty workers \citep{alistarh2018byzantine, chen2018draco,yin2018byzantine,data2020data,data2019data,data2021byzantine,pillutla2022robust,yang2023buffered,xie2020zeno++,fang2022aflguard,damaskinos2018asynchronous,ghosh2019robust,karimireddy2022byzantine,blanchard2017machine,egger2025byzantine}. 
Existing methods can be broadly categorized by the degree of synchronization required before the server updates. Most works consider synchronous round-based methods, where the server broadcasts its current iterate and waits for sufficiently many workers to return first- or zeroth-order gradient estimates computed at that iterate. Adversary-robustness is then enforced through redundancy-based techniques such as data encoding and error correction, or through server-side robust aggregation rules \citep{alistarh2018byzantine, blanchard2017machine,chen2018draco,data2020data,data2019data,data2021byzantine,yin2018byzantine,pillutla2022robust,karimireddy2022byzantine,egger2025byzantine}. 
While effective, such methods rely on coordinated worker participation and can be sensitive to heterogeneous computation times, communication delays, and stragglers. 

\begin{table}[t]
\centering
\caption{\textbf{Comparison with first-order methods}. 
All methods achieve the optimal $O(n^{-1/4})$ rate, but
\farsign{} is the only one achieving it while being adversary-robust and fully asynchronous. 
Legend: RA = robust aggregation; B = buffered; 
EF= error feedback; FDS = fixed dictionary sampling.  
$\pmark$ = worker-side asynchrony only: server waits for multiple updates before aggregation.}
\label{tab:first_order_compare}
\vspace*{\baselineskip} 
\small
\setlength{\tabcolsep}{5pt}
\begin{tabular}{lcccll}
\toprule
Method 
& \makecell{Fully\\asynchronous} 
& \makecell{Adversary\\resilience} 
& \makecell{Optimal\\$O(n^{-1/4})$ rate} 
& Technique 
& Reference \\
\midrule
\citep{ghadimi}      & \xmark      & \xmark & \cmark & SGD         & Thm.~2.1 \\
\citep{alistarh2018byzantine}        & \xmark      & \cmark & \cmark & RA          & Thm.~3.2 \\
\citep{yang2023buffered}      & $\pmark$ & \cmark & \cmark & B-RA         & Thm.~2 \\
\citep{bernsteinsignsgd}     & \xmark      & \cmark & \cmark & SignSGD     & Thm.~1 \\
\citep{karimireddy2019error}        & \xmark      & \xmark & \cmark & EF-SignSGD     & Thm.~II \\
\textbf{\farsign{} (ours)}   & \cmark      & \cmark & \cmark & \textbf{FDS-SignSGD} & \textbf{Thm.~6} \\
\bottomrule
\end{tabular}
\end{table}

Adversary-robust asynchronous learning has therefore attracted increasing attention. In first-order settings, existing approaches either screen incoming updates using empirical Lipschitz-type tests or server-side reference data \citep{damaskinos2018asynchronous,xie2020zeno++,fang2022aflguard}, or buffer asynchronous updates before applying robust aggregation \citep{yang2023buffered}. However, these methods either guarantee convergence only to a neighborhood of a stationary point \citep{damaskinos2018asynchronous}, rely on server-side reference data for update validation, or reintroduce server-side buffering, thereby limiting the utility, scalability, or responsiveness benefits of full asynchrony. By contrast, asynchronous adversary-robust zeroth-order optimization remains unexplored.

Our work addresses the above limitations and its \textbf{key highlights} are as follows:

\begin{enumerate}[leftmargin=*]
    \item \textbf{Fully asynchronous algorithmic framework}: We propose \farsign{}, a signed directional projection method for adversary-resilient optimization with both first- and zeroth-order implementations. It updates upon individual worker estimate arrivals, requiring neither server-side reference data nor buffering. It is the first such fully asynchronous adversary-resilient method.

    \item \textbf{Almost-sure convergence}: 
    Using a two-timescale stochastic recursive inclusion framework, we prove boundedness and almost-sure convergence of \farsign{} to the set of stationary points for smooth, potentially nonconvex objectives under standard zero-mean, bounded-variance noise. 
    Our two-timescale approach avoids the symmetry or unimodality assumptions often used to control sign bias in sign-SGD analyses \citep{bernstein2018signsgd}.

    \item \textbf{Convergence rates}: 
    \farsign{} achieves rates of $O(n^{-1/4+\epsilon})$ and $O(n^{-1/6+\epsilon})$ in the first- and zeroth-order regimes, respectively. 
    These match the best known rates up to the $\epsilon$-slack, even without adversaries \citep{arjevani2023lower, la2025gradient}. 
    With coupled noise, the zeroth-order rate improves to $O(n^{-1/2+\epsilon})$. 
    Tables~\ref{tab:first_order_compare}, \ref{tab:zeroth_order_compare} compare our results with existing work.
    
    \item \textbf{Empirical validation}: 
    On illustrative MNIST \citep{lecun1998mnist} simulations where the data is split near-homogeneously across workers, with $24\%$ Byzantine workers, \farsign's non-buffered nature enables it to reach $80\%$ accuracy in $11$--$37$s across attacks, whereas the fastest robust-aggregation baseline requires $223$--$420$s. FAR-SIGN also achieves comparable or better final accuracy, attaining $89.6\%$ without attack and $89.0\%$ under constant attack, while several baselines fail under Gaussian or constant attacks.
\end{enumerate}

\begin{table}[t]
\centering
\caption{\textbf{Comparison with zeroth-order methods}. While existing methods also achieve standard zeroth-order rates, 
\farsign{} is the only one achieving them while supporting both fully asynchronous execution and adversary resilience. 
Legend: DN = decoupled noise, CN = coupled noise, 
FDS = fixed dictionary sampling, and
RDS = random direction sampling.}
\label{tab:zeroth_order_compare}
\vspace*{\baselineskip} 
\small
\setlength{\tabcolsep}{2.5pt}
\makebox[\textwidth][c]{%
\begin{tabular}{l c c c c l l}
\toprule
Prior work
& \makecell{Fully\\asynchronous} 
& \makecell{Adversary\\resilience} 
& \multicolumn{2}{c}{\makecell{Convergence\\rate}} 
& Technique 
& Reference \\
\cmidrule(lr){4-5}
& & & DN & CN & & \\
\midrule
\citep{ghadimi} 
& \xmark 
& \xmark 
& -- 
& $O(n^{-1/4})$ 
& SGD 
& Thm.~3.2 \\

\citep{lian2016comprehensive} 
& \cmark 
& \xmark 
& -- 
& $O(n^{-1/4})$ 
& SGD 
& Cor.~5 \\

\citep{lian2016comprehensive} 
& \cmark 
& \xmark 
& -- 
& $O(n^{-1/2})$ 
& FDS-SGD 
& Cor.~5 \\

\citep{liu2019signsgd} 
& \xmark 
& \xmark 
& -- 
& $O(n^{-1/4})$ 
& RDS-SignSGD 
& Cor.~3 \\

\citep{bhavsar2022nonasymptotic} 
& \xmark 
& \xmark 
& $O(n^{-1/6})$ 
& $O(n^{-1/4})$ 
& SGD 
& \makecell{Thm.~1, 2} \\

\citep{egger2025byzantine} 
& \xmark 
& \cmark 
& -- 
& $O(n^{-1/4})$ 
& SGD 
& Thm.~5.7 \\

\textbf{\farsign{} (ours)} 
& \cmark 
& \cmark 
& $\mathbf{\emph{O}(n^{-1/6})}$ 
& $\mathbf{\emph{O}(n^{-1/2})}$ 
& \textbf{FDS-SignSGD} 
& \textbf{Thm.~8} \\
\bottomrule
\end{tabular}%
}
\end{table}

\section{Problem Setup, Proposed \rsign Algorithm, and Main Results} 
\label{sec:algorithmic_steps}

We now introduce our distributed-feedback setting and \farsign{}, our fully
asynchronous algorithm for adversary-resilient learning in this setting. Using two-timescale averaging and a directional robustness condition, we prove almost-sure
convergence and nonasymptotic rates that closely match the standard first- and
zeroth-order benchmarks.

\subsection{Problem Setup}
We aim to minimize a smooth, possibly nonconvex objective
$f:\mathbb{R}^d \to \mathbb{R}$ in a parameter-server architecture with $N$
workers. The server does not observe gradients or function values directly;
instead, it obtains directional feedback from workers, where a fixed but
unknown subset $\cA \subseteq [N]$ may be adversarial. Specifically, at each interaction, the server sends a point $x\in\mathbb{R}^d$ and a
direction $a$ to a worker. Honest workers return directional derivative
information, while adversarial workers may return arbitrary values. We study
two honest-worker feedback models. In the \emph{first-order} setting, an honest
worker returns $a^\top \widetilde{\nabla} f(x)$, where
$\widetilde{\nabla} f(x)$ is a stochastic gradient sample. In the
\emph{zeroth-order} setting, an honest worker does not access gradients and
instead returns the two-point finite-difference estimate $[\hat f(x + \lambda a) - \hat f(x - \lambda a)]/(2 \lambda),$ where $\lambda>0$ is the perturbation radius and $\widehat f(\cdot)$ denotes a
noisy function evaluation. Since
workers may respond at different times, feedback can arrive asynchronously and
may be computed at stale iterates.

This setup is motivated by large-scale learning systems, including
federated and parameter-server deployments, where worker-side feedback may be
delayed and be adversarially corrupted.

\subsection{Proposed \farsign Algorithm}
\label{subsec:farsign}

\farsign's pseudocode for minimizing a function $f$ is presented in Algorithm~\ref{alg:R-SIGN_Pseudocode}. Each worker $l$ is
assigned a fixed dictionary of directions
$A^{(l)}=[a_1^{(l)}\;\cdots\;a_m^{(l)}]\in\mathbb R^{d\times m}$, known to the server. The server maintains an auxiliary variable $y_n^{(l)}(i)$ for each worker-direction pair $(l,i)$, which tracks a moving average of past directional derivative estimates along $a_i^{(l)}$.

\farsign is event-driven and asynchronous: at the $n$-th update time, the server uses only the estimates currently available from the worker-direction pairs $(l,i)$ with $l\in\cW_n$ and $i\in\cI_n^{(l)}$. These estimates may have
been computed at stale iterates $x_{n-\tau_n^{(l)}}$, since the server iterate can change while a worker is computing. Honest workers return either first-order directional derivative estimates or zeroth-order finite-difference estimates, whereas adversarial workers may return arbitrary values. Upon receiving these estimates, the server updates $x_n$ along the 
directions $a_i^{(l)}$ using $\sign(-y_n^{(l)}(i))$, updates the corresponding averages $y_n^{(l)}(i)$ using $Y_i^{(l)}$, and leaves all other auxiliary variables unchanged.

For simplicity, Algorithm~\ref{alg:R-SIGN_Pseudocode} is written for the uniform-arrival setting, where every worker-direction pair $(l,i)\in[N]\times[m]$ is equally likely to produce the next available estimate. Nonuniform arrivals can be handled via inverse-probability weighting in the $x$-update.

\paragraph{Key features.}
\farsign combines event-driven asynchronous updates, signed directional descent, and two-timescale averaging. Event-driven updates allow the server to use the
worker-direction estimates available at the current update time, without waiting for a prescribed round, batch, or quorum, enabling full asynchrony. Signed
directional updates replace robust aggregation and therefore require neither a robust aggregation rule nor a trusted server-side reference dataset; they also
support both first-order directional derivatives and zeroth-order finite-difference estimates. Finally, applying the sign to the moving average $y_n^{(l)}(i)$, rather than the raw estimate $Y_i^{(l)}$, mitigates sign bias, since generally $\mathbb E[\sign(G)]\neq\sign(\mathbb E[G])$. Together with the
directional robustness condition below, these ingredients ensure that honest directional information dominates adversarial feedback.

\begin{algorithm}[t]
\caption{\farsign}
\label{alg:R-SIGN_Pseudocode}
\begin{algorithmic}[1]
    \STATE \textbf{Input:} Stepsizes $(\alpha_n),(\beta_n)$, perturbation lengths $(\lambda_n)$, direction matrix
    $A^{(l)}=[a_1^{(l)}\; a_2^{(l)}\; \ldots\; a_m^{(l)}]\in\mathbb{R}^{d\times m}$ for each worker $l\in[N]$
    \STATE \textbf{Initialize:} $x_0\in\mathbb{R}^d$, $y_0^{(l)}=0\in\mathbb{R}^m$ for each worker $l\in[N]$
    \FOR{each iteration $n\geq 0$}
        \STATE \textbf{Receive asynchronous feedback:} Let $\cW_n\subseteq[N]$ denote the workers with available responses at the $n$-th update time, and for each $l\in\cW_n$, let $\cI_n^{(l)}\subseteq[m]$ denote the directions for which estimates from worker $l$ are available. Thus, the server receives $(Y_i^{(l)}:l\in\cW_n,\ i \in \cI_n^{(l)})$.
        \STATE For an honest worker $l\in\cW_n$ and $i \in \cI_n^{(l)}$, the estimate $Y_i^{(l)}$ is one of:
        \[
        \begin{aligned}
        \text{\emph{First-order:}}\quad
        Y_i^{(l)}
        &=
        ({a_i^{(l)}})^\top \widetilde{\nabla} f(x_{n-\tau_n^{(l)}}),
        \\[0.75ex]
        \text{\emph{Zeroth-order:}}\quad
        Y_i^{(l)}
        &=
        \frac{
        \widehat f(x_{n-\tau_n^{(l)}}+\lambda_n a_i^{(l)})
        -
        \widehat f(x_{n-\tau_n^{(l)}}-\lambda_n a_i^{(l)})
        }{2\lambda_n}.
        \end{aligned}
        \]
        \STATE If worker $l$ is adversarial, $Y_i^{(l)}$ may be arbitrary. 
        \STATE \textbf{Signed directional update:}
        \[
        x_{n+1}
        =
        x_n
        +
        \alpha_n
        \sum_{l\in\cW_n}
        \sum_{i \in \cI_n^{(l)}}
        a_i^{(l)}\,\sign\!\left(-y_n^{(l)}(i)\right).
        \]
        \STATE Send $x_{n+1}$ to each worker $l\in\cW_n$.
        \STATE \textbf{Update directional averages:}
        For each $l\in[N]$ and $i\in[m]$,
        \[
        y_{n+1}^{(l)}(i)
        =
        \begin{cases}
        y_n^{(l)}(i)+\beta_n\bigl(Y_i^{(l)}-y_n^{(l)}(i)\bigr),
        & l\in\cW_n,\ i\in \cI_n^{(l)},\\
        y_n^{(l)}(i),
        & \text{otherwise.}
        \end{cases}
        \]
    \ENDFOR
\end{algorithmic}
\end{algorithm}

\subsection{Main Results: Asymptotic Convergence, and Convergence Rates}
We impose the following assumptions throughout.

To obtain explicit constants, we assume that $|\cW_n| = |\cI_n^{(l)}| = 1$. That is, at each update $n \geq 0$, the estimate of a single randomly chosen worker along a single randomly selected direction $a_i^{(l)}$ is available. The results extend to other update patterns with suitable modifications of the constants.

\begin{assumption}[\textbf{Objective function}]
    \label{L-Smooth}
    The function $f: \bR^d \to \bR$ is twice continuously differentiable, coercive, and $L$-smooth, i.e., $\exists L > 0$ such that $    \norm{\nabla f(x)-\nabla f(y)} \leq L \norm{x-y} \quad \forall x , y \in \mathbb{R}^d.$
\end{assumption}

We next specify the stochastic model for stale feedback. At each update time
$n$, every worker $w\in[N]$ has a potential staleness value $\tau_n^{(w)}$.
If worker $w$ contributes feedback at time $n$, then this feedback is computed
at the stale iterate $x_{n-\tau_n^{(w)}}$. Let $\cF_n$ denote the history before
the current responding worker-direction pairs are revealed:
\begin{equation}
\label{e:filtration.defn}
    \cF_n
    :=
    \sigma\!\left(
    \{x_k, y_k^{(w)}, \tau_k^{(w)}:0\le k\le n,\ w\in[N]\}
    \right).
\end{equation}
Conditional on $\cF_n$, the responding worker-direction pairs at time $n$ are
independent of the fresh oracle noise. 

    \begin{assumption}[\textbf{First-order feedback noise}]
    \label{noise-first order}
        For every honest worker $w \in[N]$ and $n\ge0$, its noisy gradient estimate at     $x_{n-\tau_n^{(w)}}$ satisfies
        \[
        \mathbb{E}\!\left[    \widetilde{\nabla}f(x_{n-\tau_n^{(w)}})\mid\cF_n     \right]     \overset{\as}{=}     \nabla f(x_{n-\tau_n^{(w)}}),\qquad     \mathbb{E}\!\left[\left\|   \widetilde{\nabla}f(x_{n-\tau_n^{(w)}}) - \nabla f(x_{n-\tau_n^{(w)}}) \right\|^2    \mid\cF_n    \right] \le \sigma^2,
    \]
    for some $\sigma^2>0$. The dependence of
    $\widetilde{\nabla}f(\cdot)$ on worker $w$ is suppressed for notational
    convenience.
    \end{assumption}

    \begin{assumption}[\textbf{Zeroth-order feedback noise}]
    \label{noise-assump}
        For every honest worker $w \in[N]$, dictionary direction $a_i^{(w)}$ with $i\in[m]$, and $n\ge0$, the noisy function evaluations at $x_{n-\tau_n^{(w)}}$ satisfy
        \[
            \begin{aligned}
            \widehat f(x_{n-\tau_n^{(w)}}+\lambda_n a_i^{(w)})
            &=
            f(x_{n-\tau_n^{(w)}}+\lambda_n a_i^{(w)})+\zeta_{n+1}^{(1)},\\
            \widehat f(x_{n-\tau_n^{(w)}}-\lambda_n a_i^{(w)})
            &=
            f(x_{n-\tau_n^{(w)}}-\lambda_n a_i^{(w)})+\zeta_{n+1}^{(2)}.
            \end{aligned}
            \]
            Moreover, for $r\in\{1,2\}$, we have $\mathbb E\!\left[\zeta_{n+1}^{(r)}\mid\cF_n\right]\overset{\as}{=}0,$ and $
            \mathbb E\!\left[|\zeta_{n+1}^{(r)}|^2\mid\cF_n\right]\le K$ 
        for some $K>0$. The dependence of $\widehat f(\cdot)$ and the noise variables
        on worker $w$ is suppressed for notational convenience.
    \end{assumption}

\begin{assumption}[\textbf{Stepsizes}]
\label{Step-sizeA}
The stepsizes $(\alpha_n)$ and $(\beta_n)$ satisfy
\[
    \sum_n\alpha_n=\sum_n\beta_n=\infty,\qquad
    \sum_n(\alpha_n^2+\beta_n^2)<\infty,\qquad
    \frac{\alpha_n}{\beta_n}\to0,\qquad
    \sum_n\frac{\alpha_n^2}{\beta_n}<\infty .
\]
For the zeroth-order case, the perturbation lengths $(\lambda_n)$ also 
satisfy $\sum\limits_n\dfrac{\beta_n^2}{\lambda_n^2}<\infty,$ and $\sum\limits_n\beta_n\lambda_n<\infty.$
\end{assumption}

\begin{remark}
An example of stepsize and perturbation  sequences satisfying the above condition is
\[
    \alpha_n = \frac{1}{(n + 1)^a}, \quad \beta_n = \frac{1}{(n + 1)^b}, \quad  \text{ and } \lambda_n  = \frac{1}{(n + 1)^p},
\]
where $\frac{1}{2}< b < a \leq 1,$\  $b + p > 1,$ $b - p > \frac{1}{2},$ and $2a - b > 1,$ e.g., $a = 0.91, b = 0.8, p = 0.25.$
\end{remark}

\begin{assumption}[\textbf{$L^2$-dominated staleness}]
\label{ass:delay}
For each worker $w\in[N]$, the staleness sequence $(\tau_n^{(w)})$ is
nonnegative and integer-valued. Moreover, there exists a nonnegative random
variable $\tau$ with $\mathbb E[\tau^2]<\infty$ such that $\tau_n^{(w)}\le \tau  \text{ a.s. for all } n\ge1,\; w\in[N].$
\end{assumption}




\begin{assumption}
    The matrix $A$ has full column rank. Moreover, for any subset $S$  satisfying $|S| \leq |\mathcal{A}|$, 
    \begin{equation*}
        \sum\limits_{w \in S^c} \norm{ ({A^{(w)}})^\top x}_1 > \sum\limits_{w \in S} \norm{ (A^{(w)})^\top x}_1 \quad \forall \; x \in \mathbb{R}^d\setminus\{0\},
    \end{equation*}
    where $S^c$ denotes the complement of $S$. This condition is agnostic to the objective function $f.$
    \label{fawzi-condn}
\end{assumption}

\begin{remark}
Assumption~\ref{fawzi-condn} holds for several simple and practically relevant observation designs,. Two representative cases are as follows.
\begin{enumerate}[leftmargin=*]
    \item \textbf{Full-coordinate observations.}
    If $A^{(w)} = \bI_d,$ the $d$-dimensional identity matrix for all $w$, then Assumption~\ref{fawzi-condn} holds whenever $|\mathcal{A}|/N < 1/2$. This choice reduces \farsign to a two-timescale variant of stochastic sign-SGD, at an $O(Nd)$ memory cost.

    \item \textbf{Scalar (single-direction) observations.}
    When each worker uses a single probing direction ($m=1$), Assumption~\ref{fawzi-condn} reduces to the robustness condition in \citep{fawzi2014secure,ganesh2023online}. In this case, each $y_n^{(w)}$ is scalar, yielding an $O(N)$ memory footprint. For example, in the setting with four-workers and one adversarial settworker, the  following set of direction matrices from \citep{ganesh2023online} satisfies Assumption~\ref{fawzi-condn}:
    \[
        A^{(1)} = \begin{bmatrix} 2 \\ 0 \end{bmatrix},\quad
        A^{(2)} = \begin{bmatrix} 0 \\ 2 \end{bmatrix},\quad
        A^{(3)} = \begin{bmatrix} 1 \\ 2 \end{bmatrix},\quad
        A^{(4)} = \begin{bmatrix} -2 \\ 1 \end{bmatrix}.
    \]
\end{enumerate}
\end{remark}

Our main result on \farsign's asymptotic convergence can be stated as follows.



\begin{theorem}[Almost-sure convergence]
\label{thm:almost.sure.conv}
Let $(x_n)$ denote the sequence of iterates generated by
Algorithm~\ref{alg:R-SIGN_Pseudocode}. Suppose Assumptions
\ref{L-Smooth}, \ref{Step-sizeA}, \ref{ass:delay}, and \ref{fawzi-condn} hold. In addition, assume that either the feedback is first-order satisfying Assumption~\ref{noise-first order}, or zeroth-order satisfying Assumption~\ref{noise-assump}. Then, the sequence $(x_n)$ converges almost surely to $\mathcal{X}^\star \;=\; \{x \in \mathbb{R}^d : \nabla f(x) = 0\},$ the set of stationary points of $f$.
\end{theorem}

\begin{remark}
Theorem~\ref{thm:almost.sure.conv} gives the standard stationarity guarantee
for smooth nonconvex objectives in a setting that combines adversarial workers,
asynchronous stale feedback, and sign-based directional updates. The
two-timescale averaging is essential: unlike vanilla sign-SGD, which can fail
under asymmetric noise even without adversaries \citep{karimireddy2019error},
\farsign requires only standard martingale noise and bounded second moments.
\end{remark}

\begin{theorem}[First-order rate]
\label{thm:rate_optimalfirstorder}
Suppose Assumptions~\ref{L-Smooth}, \ref{noise-first order}, \ref{ass:delay}, and \ref{fawzi-condn} hold. Run \farsign with
$\alpha_n = \frac{1}{mN} (n+1)^{-3/4-\epsilon},$ and
$\beta_n =\frac{1}{mN} (n+1)^{-1/2}$, where $0<\epsilon<1/4$. Then, there exists some constant $C>0$ such that, for all
$n\ge1$,
\[
    \frac{\sum_{k=1}^n \alpha_k\,\mathbb E\|\nabla f(x_k)\|_1}
    {\sum_{k=1}^n \alpha_k}
    \le
    C m^{3/2}N (n+1)^{-1/4+\epsilon}.
\]
\end{theorem}

%
\begin{remark}
    With first-order feedback, \farsign achieves a near-optimal nonconvex stochastic rate, despite adversarial workers and asynchronous stale feedback.
\end{remark}

\begin{theorem}[Zeroth-order rates]
\label{thm:rate_optimal}
Suppose Assumptions~\ref{L-Smooth}, \ref{noise-assump}, \ref{ass:delay},
and \ref{fawzi-condn} hold, and let $\lambda>0$. Then, for some constant $C>0,$ the following statement hold for all $n\ge1$:
\begin{enumerate}[leftmargin=*]
    \item If
    $\alpha_n = \frac{1}{mN}(n+1)^{-5/6-\epsilon}$,
    $\beta_n =\frac{1}{mN} (n+1)^{-2/3}$, and
    $\lambda_n=\lambda(n+1)^{-1/6}$, with $0<\epsilon<1/6$ and $\lambda >0$ then
    \[
        \frac{\sum_{k=1}^n \alpha_k\,\mathbb E\|\nabla f(x_k)\|_1}
        {\sum_{k=1}^n \alpha_k}
        \le
        C m^{3/2}N (n+1)^{-1/6+\epsilon}.
    \]

    \item Additionally, if the two function-evaluation noises are coupled, i.e.,
    $\zeta_{n+1}^{(1)}=\zeta_{n+1}^{(2)}$, and
    $\alpha_n = \frac{1}{mN}(n+1)^{-1/2-\epsilon_1}$,
    $\beta_n = \frac{1}{mN} (n+1)^{-\epsilon_2}$,
    $\lambda_n=\lambda(n+1)^{-1/2}$, with
    $0<\epsilon_2<\epsilon_1<1/2$, then
    \[
        \frac{\sum_{k=1}^n \alpha_k\,\mathbb E\|\nabla f(x_k)\|_1}
        {\sum_{k=1}^n \alpha_k}
        \le
        C m^{3/2}N (n+1)^{-1/2+\epsilon_1}.
    \]
\end{enumerate}
\end{theorem}

\begin{remark}
    \farsign attains the standard zeroth-order nonconvex rate under independent evaluation noise and recovers the faster rate when the two function evaluations share the same noise.
\end{remark}

\section{Outline of Proofs of Main Results}
\label{sec:3}

We now outline the proofs of our main results, emphasizing the technical ideas
behind \farsign{} and how they enable full asynchrony and adversary resilience.
At a high level, the analysis combines two-timescale tracking of stale
directional feedback with a directional robustness condition that guarantees
descent despite adversarial workers. Full details are deferred to the appendix.

\subsection{Proof of Theorem~\ref{thm:almost.sure.conv}}
The proof proceeds in four steps.

\emph{Step 1 (Fast-timescale tracking).}
The condition $\alpha_n/\beta_n\to0$ makes $x_n$ quasi-static from the
perspective of the faster $y$-recursion. Thus, for each honest worker
$w\in\cA^c$, $y_n^{(w)}$ tracks the moving equilibrium
${A^{(w)}}^\top\nabla f(x_n)$. With
$e_n^{(w)}:=y_n^{(w)}-{A^{(w)}}^\top\nabla f(x_n)$, we prove
$\|e_n^{(w)}\|\to0$ and
$\sum_n\alpha_n\|e_n^{(w)}\|<\infty$ a.s. This controls sign bias and provides
the summable error term needed below.

\emph{Step 2 (Robust  descent).}
Conditioning on the current state, the mean slow drift is
\[
\bar g_n
=
\sum_{w\in\cA}A^{(w)}\xi_n^{(w)}
+
\sum_{w\in\cA^c}A^{(w)}\sign(-y_n^{(w)}),
\qquad
\xi_n^{(w)}\in[-1,1]^m .
\]
Write $\bar g_n=g_{\rm id}(x_n)+r_n$, where $g_{\rm id}(x_n)$ is obtained by
replacing $y_n^{(w)}$ with ${A^{(w)}}^\top\nabla f(x_n)$ for honest workers,
and $r_n$ is the corresponding tracking-error residual. Assumption~\ref{fawzi-condn},
following the robustness idea from \citep{ganesh2023online}, gives
$\nabla f(x_n)^\top g_{\rm id}(x_n)\le-\eta\|\nabla f(x_n)\|_1$ for $\eta > 0$.
Since $a(\sign(-b)-\sign(-a))\le2|a-b|$ for $a, b \in \bR,$ we get
$\nabla f(x_n)^\top r_n
\le2\sum_{w\in\cA^c}\|e_n^{(w)}\|_1$.
Thus, adversaries are controlled by directional robustness, while the residual
is summable by Step~1.

\emph{Step 3 (Stability).}
Combining the above drift bound with $L$-smoothness yields an
almost-supermartingale inequality for $f(x_n)$. Since the positive error terms
are summable by Step~1 and Assumption~\ref{Step-sizeA}, the
Robbins--Siegmund theorem \citep{robbins1971convergence} implies that
$f(x_n)$ converges a.s. Coercivity of $f$ then gives almost-sure boundedness of
$(x_n)$.

\emph{Step 4 (Limiting stochastic recursive inclusion).}
The slow recursion is naturally set-valued because the sign map is
discontinuous at zero and adversaries may choose arbitrary signed directions.
Using boundedness and Step~1, the two-timescale stochastic recursive inclusion
framework of \citet{yaji2020stochastic} shows that $x_n$ tracks
$\dot x(t)\in H(x(t))$, where
\[
H(x)
=
\left\{
\sum_{w\in\cA}A^{(w)}\xi^{(w)}
+
\sum_{w\in\cA^c}A^{(w)}\lambda^{(w)}
:
\xi^{(w)}\in[-1,1]^m,\ 
\lambda_j^{(w)}\in\Sign\!\left(-{a_j^{(w)}}^\top\nabla f(x)\right)
\right\},
\]
with $\Sign(z)=\{\sign(z)\}$ for $z\neq0$ and $\Sign(0)=[-1,1]$. We show that
$H$ is Marchaud and that $f$ is a Lyapunov function: Assumption~\ref{fawzi-condn}
implies $\nabla f(x)^\top v\le-\eta\|\nabla f(x)\|_1$ for every $v\in H(x)$.
Hence the largest invariant set is contained in
$\mathcal X^\star=\{x:\nabla f(x)=0\}$. The convergence theorem of
\citet{yaji2020stochastic} yields $x_n\to\mathcal X^\star$ a.s.

\subsection{ Outline of Proof of Theorem 2.7}
The proof establishes explicit non-asymptotic convergence rates for the zeroth-order setting and proceeds in four steps.

\emph{Step 1 (Finite-horizon tracking error)}. We first establish a conditional mean-square bound for the fast-timescale tracking error $e_n^{(w)} := y_n^{(w)} - {A^{(w)}}^\top \nabla f(x_n)$. We observe that $\{e_n^{(w)} \}$ satisfies a linear stochastic approximation recursion, perturbed by a delay element and a bias term of $\mathcal{O}(\alpha_n)$ arising from the time-varying nature of $x_n$. Adapting standard stochastic approximation techniques and  incorporating the above non-trivialities, we derive the conditional bound $\mathbb{E}[\|e_{n+1}^{(w)}\|^2 \mid \mathcal{F}_n] \le (1-\beta_n) \mathbb{E}[\|e_n^{(w)}\|^2] + R_n + \mathbb{E}[D_n]$, where $R_n = \mathcal{O}\big(\frac{\beta_n^2}{\lambda_n^2} + \beta_n \lambda_n^2 + \alpha_n^2 + \beta_n^2 + \frac{\alpha_n^2}{\beta_n}\big)$, and the delay component is $D_n = \mathcal{O}\big(\tau_n^{(w)} \sum_{k=n-\tau_n^{(w)}}^{n-1} \alpha_k^2\big)$ (see Proposition \ref{propoenbound 1} for the detailed derivation). Using Fubini's Theorem in Lemma \ref{lem:delay-summable}, we establish that: $\mathbb{E}[\sum_{k=1}^{n} D_k] = \mathcal{O} \big(\mathbb{E}[\tau^2]\sum_{k=1}^{n} \alpha_k^2\big)$.

\emph{Step 2 (Tracking Rate)}. Next, by iteratively expanding the above recursion and applying the discrete finite-horizon estimate from Proposition \ref{Prop:Bigo}, we obtain in Corollary  \ref{mean square tracking raye}: $\mathbb{E}\big[\|e_{n+1}^{(w)}\|^2\big] \le \widetilde{C}_1\,(1+n)^{-\gamma}$, where $\widetilde{C}_1 > 0$ is a constant and  $\gamma = \min \{b-2p,\, 2p,\, 2a-b,\, 2a-2b\}$.

\emph{Step 3 (Ergodic gradient bound)}. We then leverage the standard $L$-smoothness descent lemma. By incorporating the properties of the signed update --- following the same approach detailed in Step 2 of Theorem \ref{thm:almost.sure.conv} --- we obtain the inequality:
\[
\begin{aligned}
\mathbb{E}[f(x_{n+1}) \mid \mathcal{F}_n]
\le{}&
f(x_n)
- \alpha_n \eta \|\nabla f(x_n)\|_1
+ 2\alpha_n |\mathcal{A}^c| \|e_n^{(w)}\|_1
+ \frac{L}{2}\alpha_n^2 K_5^2 .
\end{aligned}
\]
A key non-triviality in evaluating this descent is the presence of the two-timescale tracking error term $2\alpha_n |\mathcal{A}^c| \|e_n^{(w)}\|_1$. We control this term by transitioning from the $L_2$ norm bounds established in Step 2 to the expected $L_1$ norm using Jensen's inequality ($\mathbb{E}\|e_k^{(w)}\| \le \sqrt{\mathbb{E}\|e_k^{(w)}\|^2}$) and the norm equivalence ($\|v\|_1 \le \sqrt{m}\|v\|_2$).

\emph{Step 4 (Final Convergence Rate)}. Finally, by taking the total expectation, telescoping the descent inequality over $n$ iterations, and dividing by $S_n := \sum_{k=1}^n \alpha_k$, we obtain the ergodic gradient bound (Corollary \ref{cor:ergodic_constants}). Substituting the step-size given in Theorem \ref{thm:rate_optimal} we obtain the desired result.

\section{Numerical Simulation}
\label{sec:numericalsec}

We empirically evaluate whether \farsign{}'s non-buffered signed-directional updates translate into wall-clock gains in adversarial neural-network training. 
Our focus is not to benchmark large-scale image classification, but to provide a controlled comparison against Byzantine-resilient zeroth-order robust-aggregation baselines under the same query budget.

\paragraph{Experimental setup.}

Our testbed is MNIST classification \citep{lecun1998mnist} with 60,000 training and 10,000 test samples using a two-layer MLP 
($784 \to 100 \to 10$, $d=79{,}510$ parameters) in a parameter-server architecture. 
We use $N=51$ workers, of which $12$ are adversarial, corresponding to $24\%$ adversarial worker fraction. 
The data is split randomly across workers.

We evaluate \farsign{} against CYBER-0~\citep{neto2024communicationefficientbyzantineresilientfederatedzeroorder}, a Byzantine-resilient zeroth-order federated optimization method and our closest baseline. 
We combine CYBER-0 with buffered execution~\citep{yang2023buffered} and standard robust aggregation rules: Krum~\citep{blanchard2017machine}, Multi-Krum~\citep{blanchard2017machine}, RFA~\citep{pillutla2022robust}, Median~\citep{yin2018byzantine}, Trimmed Mean~\citep{yin2018byzantine}, and Bulyan~\citep{mhamdi2018hidden}. 
We consider four Byzantine attacks: Sign-Flip~\citep{damaskinos2018asynchronous}, Constant, Gaussian~\citep{blanchard2017machine}, and ALIE~\citep{baruch2019little}. 
No communication delays or stale gradient estimates are imposed: at each iteration, a worker is selected at random, computes an estimate at the current server iterate, and returns it to the server. 
All methods use mini-batches of size $64$, perturbation radius $\lambda=0.001$, and learning rates multiplied by a factor of $0.99$ every $100$ iterations. 
Each method is run for $1{,}280{,}000$ zeroth-order calls, and results are averaged over $5$ random seeds.

For the zeroth-order implementation of \farsign{}, we use identity observation matrices, $A^{(l)}=I_d$, so each selected worker estimates coordinate-wise directional derivatives directly. 
The worker samples $K=64$ scalar entries of the flattened network-parameter vector and computes two-point finite-difference estimates along the corresponding standard basis directions. 
For each perturbation $x_n \pm \lambda_n e_i$, the flat parameter vector is reshaped into the MLP's layer-wise weights and biases, and $\hat f(x_n \pm \lambda_n e_i)$ is evaluated as the mini-batch cross-entropy loss from a standard forward pass on the same mini-batch. 
These estimates update the fast-timescale auxiliary variable, after which the server immediately applies a sparse signed update supported on the selected coordinates. 
For \farsign{}, we initialize the two stepsizes as $(\alpha,\beta)=(0.1,0.2)$.

For the buffered CYBER-0 baselines, workers sample random directions from the unit sphere, compute two-point directional derivatives, and reconstruct dense gradients before robust aggregation. 
The workers are partitioned into $B=25$ fixed buffer slots, with two to three workers per slot. 
The server stores incoming estimates in their corresponding slots and updates only after every slot has at least one estimate, at which point it applies a standard robust aggregation and updates the model. 
The buffer size satisfies the BASGD requirement $B \geq 2Nf+1$ for tolerating $\lfloor Nf \rfloor = 12$ adversaries. 
Thus, the wall-clock comparison isolates the structural cost of dense gradient reconstruction, buffering, and robust aggregation, rather than any effect of artificial delays or stale gradients. 
For the CYBER-0 baselines, we use learning rate $\gamma=0.2$.

\paragraph{Results.}

\begin{figure}[t]
    \centering
    \makebox[\textwidth][c]{%
        \includegraphics[
            width=\textwidth]{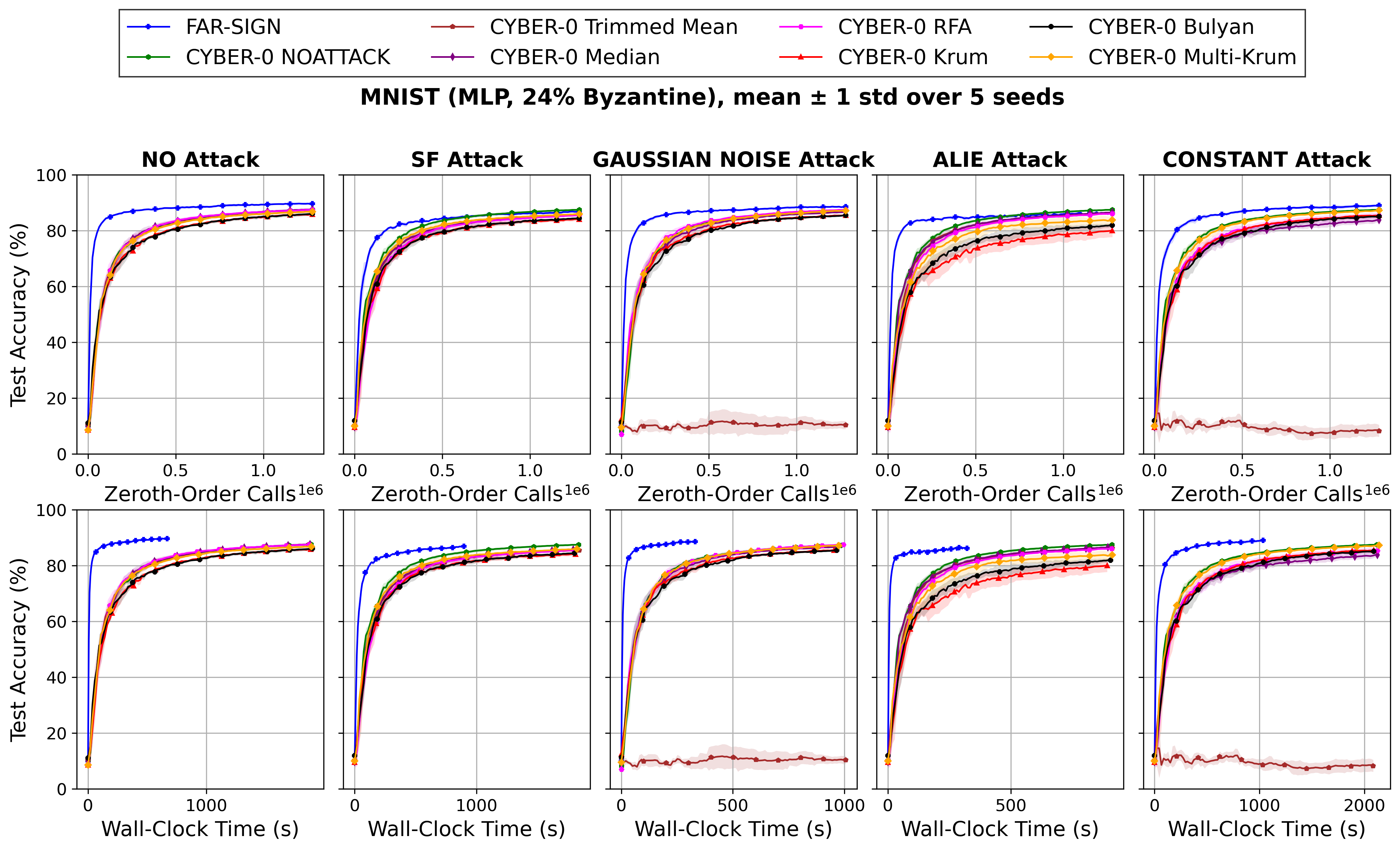}%
    }
    \vspace{0.6em}
    \caption{MNIST classification under no attack and four adversarial attacks. \farsign{} shows comparable zeroth-order sample efficiency, but substantially better wall-clock efficiency.}
    \label{fig:wall_clock_nn_gn_alie}
\end{figure}

Figure~\ref{fig:wall_clock_nn_gn_alie} compares FAR-SIGN against CYBER-0 combined with buffered robust aggregation. Each column corresponds to an independent cloud run (see Section~\ref{expsetup}), and all reported curves are averaged over five random seeds. For every seed, the MNIST dataset was randomly repartitioned across workers. The results show that FAR-SIGN consistently matches or outperforms the baselines in terms of accuracy versus the number of zeroth-order calls. However, because FAR-SIGN updates immediately upon receiving worker feedback and does not require server-side waiting or buffering, it achieves substantially better wall-clock performance.

For Table~\ref{tab:time_nn_24pct}, all experiments were conducted on the same local machine (see Section~\ref{expsetup}) to ensure a uniform comparison across algorithms and attack scenarios. The results again show that \farsign{} consistently improves wall-clock performance while preserving accuracy. 

Without Byzantine workers, \farsign{} reaches $85\%$ accuracy in $24$s, whereas the fastest CYBER-0 baseline requires $417$s, yielding a $17.4\times$ speedup. It also attains the highest maximum accuracy, $89.6\pm0.4\%$. The advantage persists under attacks: under ALIE and Gaussian attacks, \farsign{} reaches $85\%$ accuracy in $70$s and $35$s, respectively, while the fastest CYBER-0 baselines require $604$s and $446$s. Overall, \farsign{} achieves $6.2$--$35.2\times$ speedups across scenarios. 

These gains again stem from the algorithmic distinction highlighted earlier: \farsign{} updates immediately using sparse coordinate-wise operations on only $K=64$ selected coordinates, whereas CYBER-0-style baselines construct high-dimensional perturbations, reconstruct dense gradients in dimension $d=79{,}510$, and wait for buffered robust aggregation. 

We also observe that several baselines degrade substantially under stronger attacks. In particular, CYBER-0 Trimmed Mean collapses under Gaussian and Constant attacks, achieving only $10.4\pm1.4\%$ and $8.3\pm2.2\%$ accuracy, respectively. More broadly, several baselines fail to reach $85\%$ accuracy under ALIE, Constant, or Sign-Flip attacks, whereas \farsign{} consistently maintains $86.2$--$89.0\%$ accuracy under attacks and $89.6\pm0.4\%$ without attacks. The top row of Figure~\ref{fig:wall_clock_nn_gn_alie} shows comparable zeroth-order sample efficiency, while the bottom row highlights the substantial wall-clock advantage of \farsign{}.

\begin{table*}[t]
\centering
\scriptsize
\caption{Wall-clock time (seconds) to reach target accuracy (MLP, MNIST).
24\% Byzantine workers, 51 clients. Results averaged over 5 seeds.
Bold indicates best performance. ``--'' denotes failure to reach target accuracy and Max acc. denotes Max accuracy.}
\label{tab:time_nn_24pct}
\vspace*{\baselineskip}
\setlength{\tabcolsep}{1.8pt}
\renewcommand{\arraystretch}{1.0}
\resizebox{\textwidth}{!}{%
\begin{tabular}{l ccc ccc ccc ccc ccc}
\toprule
\textbf{Method}
& \multicolumn{15}{c}{\textbf{Attack}} \\
\cmidrule(lr){2-16}
& \multicolumn{3}{c}{No-Attack}
& \multicolumn{3}{c}{ALIE}
& \multicolumn{3}{c}{Gaussian}
& \multicolumn{3}{c}{Constant}
& \multicolumn{3}{c}{Sign-Flip} \\
\cmidrule(lr){2-4}
\cmidrule(lr){5-7}
\cmidrule(lr){8-10}
\cmidrule(lr){11-13}
\cmidrule(lr){14-16}
& 80\% & 85\% & {\tiny Max acc.}
& 80\% & 85\% & {\tiny Max acc.}
& 80\% & 85\% & {\tiny Max acc.}
& 80\% & 85\% & {\tiny Max acc.}
& 80\% & 85\% & {\tiny Max acc.} \\
\midrule
FAR-SIGN
& \textbf{11} & \textbf{24} & \textbf{89.6$\pm$0.4}
& \textbf{19} & \textbf{70} & 86.2$\pm$0.2
& \textbf{15} & \textbf{35} & \textbf{88.6$\pm$0.2}
& \textbf{24} & \textbf{51} & \textbf{89.0$\pm$0.4}
& \textbf{37} & \textbf{120} & \textbf{86.9$\pm$0.3} \\
RFA
& 223 & 417 & 87.5$\pm$0.3
& 285 & 604 & 86.0$\pm$0.2
& 205 & 446 & 87.5$\pm$0.3
& 312 & 839 & 85.4$\pm$1.2
& 262 & 697 & 85.5$\pm$0.4 \\
Multi-Krum
& 288 & 594 & 86.9$\pm$0.2
& 268 & 634 & 83.8$\pm$1.6
& 236 & 599 & 87.2$\pm$0.3
& 249 & 480 & 87.2$\pm$0.3
& 233 & 566 & 86.0$\pm$0.4 \\
Median
& 247 & 467 & 87.6$\pm$0.1
& 280 & 608 & \textbf{86.6$\pm$0.3}
& 293 & 542 & 86.8$\pm$0.5
& 350 & -- & 83.8$\pm$0.9
& 280 & 723 & 85.7$\pm$0.3 \\
Trimmed Mean
& 260 & 511 & 87.5$\pm$0.3
& 266 & 604 & 86.2$\pm$0.2
& -- & -- & 10.4$\pm$1.4
& -- & -- & 8.3$\pm$2.2
& 257 & 649 & 85.5$\pm$0.3 \\
Krum
& 314 & 803 & 85.8$\pm$0.3
& 643 & -- & 80.0$\pm$2.3
& 303 & 867 & 85.5$\pm$0.2
& 331 & -- & 85.5$\pm$0.5
& 372 & -- & 84.1$\pm$0.6 \\
Bulyan
& 374 & 772 & 85.8$\pm$0.5
& 408 & -- & 81.9$\pm$0.7
& 321 & 747 & 85.5$\pm$0.3
& 420 & 885 & 85.1$\pm$0.4
& 372 & 872 & 84.4$\pm$0.5 \\
\bottomrule
\end{tabular}%
}
\end{table*}

\section{Conclusion and Future Work}
\label{sec:5}
We proposed \farsign{}, a signed directional projection framework for adversary-resilient optimization with both first- and zeroth-order feedback. 
\farsign{} updates upon individual worker estimate arrivals, requiring neither server-side reference data nor buffering. 
Using a two-timescale stochastic recursive inclusion framework, we proved boundedness and almost-sure convergence to stationary points for smooth, potentially nonconvex objectives, and established rates of $O(n^{-1/4+\epsilon})$ and $O(n^{-1/6+\epsilon})$ in the first- and zeroth-order regimes, respectively. 
Our MNIST experiments with $24\%$ Byzantine workers show that this non-buffered structure yields $6.2$--$35.2\times$ wall-clock speedups over CYBER-0 robust-aggregation baselines while maintaining comparable or better accuracy.

The current work assumes homogeneous worker objectives, with all workers evaluating gradients of the same function $f$. Extending the theory to heterogeneous objectives and evaluating \farsign{} on larger architectures and federated benchmarks are important directions for future work.

\bibliographystyle{plainnat}
\bibliography{references}

\newpage

\appendix

\section{Math Preliminary}
\begin{theorem}[ {\cite{robbins1971convergence}}]
\label{Robbins Siegmund}
Let $\{X_n\}$, $\{Y_n\}$, and $\{Z_n\}$ be nonnegative random processes adapted to
a filtration $\{\mathcal{F}_n\}$. Suppose that
\[
    \mathbb{E}[X_{n+1} \mid \mathcal{F}_n]
    \le
    (1 + a_n) X_n - Y_n + Z_n
    \qquad \text{a.s.},
\]
where $\{a_n\}$ is a deterministic sequence satisfying
\[
    \sum_{n \ge 1} a_n < \infty,
\]
and assume further that
\[
    \sum_{n \ge 1} Z_n < \infty
    \qquad \text{almost surely}.
\]
Then $\{X_n\}$ converges almost surely to a finite nonnegative random variable,
and
\[
    \sum_{n \ge 1} Y_n < \infty
    \qquad \text{almost surely}.
\]
\end{theorem}

\begin{proposition}
\label{prop:Sn_bound}
Let $\beta_n = C_b(n+1)^{-b}$ with $b \in (0,1)$ and let
$P_n = C_p(1+n)^{-p}$ for some constants $C>0$ and $p>0$. Define
\[
    S_n
    :=
    \sum_{i=0}^{n}
    \Bigg(
        \prod_{k=i+1}^{n} (1-\beta_k)
    \Bigg) P_i .
\]
Then
\begin{equation}
\label{eq:Sn_final}
    S_n
    \le
    \frac{3 C_p\,2^p (n+1)^b}{C_b (n+2)^p}
    + o\!\left(\frac{1}{(n+1)^r}\right),
    \qquad \text{for every } r>0 .
\end{equation}
\label{Prop:Bigo}
\end{proposition}

\begin{remark}
The bound in Proposition~\ref{prop:Sn_bound} admits the following interpretation.
The leading term is of order
\[
    \mathcal{O}\!\left((n+1)^{-(p-b)}\right),
\]
while the remainder term is super-polynomially small, i.e., it decays faster
than any inverse polynomial. Consequently, the estimate separates the dominant
polynomial decay from an exponentially small tail.
\end{remark}

\begin{proof}
We start by noting that
 
 \[
\begin{aligned}
\sum_{i=0}^{n}
\Bigg(
\prod_{k=i+1}^{n} (1-\beta_k)
\Bigg) P_i
\le
\sum_{i=0}^{n}
\exp\!\left(
- \sum_{k=i+1}^{n} \beta_k
\right)
\frac{C}{(1+i)^p}.
\end{aligned}
\]
With \(\beta_k=C_b(k+1)^{-b}\) and \(0<b<1\),
\[
\sum_{k=i+1}^n \beta_k
= C_b\sum_{k=i+1}^n (k+1)^{-b}
\ge C_b\frac{n-i}{(n+1)^b}.
\]
Hence
\[
\exp\!\Big(-\sum_{k=i+1}^n\beta_k\Big)
\le
\exp\!\Big(-C_b\frac{n-i}{(n+1)^b}\Big),
\]
and, with the change of variable \(m=n-i\),
\[
\begin{aligned}
S_n
&\le
\sum_{m=0}^n
\exp\!\Big(-C_b\frac{m}{(n+1)^b}\Big)
\frac{C}{(n-m+1)^p}
\\
&=
\sum_{m=0}^{\lfloor n/2\rfloor}
\exp\!\Big(-C_b\frac{m}{(n+1)^b}\Big)
\frac{C}{(n-m+1)^p}
\\
&\quad+
\sum_{m=\lfloor n/2\rfloor+1}^{n}
\exp\!\Big(-C_b\frac{m}{(n+1)^b}\Big)
\frac{C}{(n-m+1)^p}
\\
&=: \mathbf T_1 + \mathbf T_2.
\end{aligned}
\]
For \(\mathbf T_1\) we have \(n-m+1\ge (n+2)/2\), so
\[
\begin{aligned}
\mathbf T_1
&\le
\frac{C\,2^p}{(n+2)^p}
\sum_{m=0}^{\lfloor n/2\rfloor}
\exp\!\Big(-C_b\frac{m}{(n+1)^b}\Big)
\\
&\le
\frac{C\,2^p}{(n+2)^p}
\frac{1}{1-\exp\!\big(-C_b/(n+1)^b\big)}. 
\end{aligned}
\]
Using \(1-e^{-x}\ge x/3\) for \(0<x<1\) with \(x=C_b/(n+1)^b\) (valid for large \(n\)), we get
\[
\mathbf T_1
\le
\frac{C\,2^p}{(n+2)^p}\cdot\frac{3(n+1)^b}{C_b}
\le
\frac{3\cdot 2^p}{C_b}\,C\,(n+1)^{b-p}.
\]
For \(\mathbf{T}_2\) we have
\[
\begin{aligned}
&\sum_{m=\lfloor n/2 \rfloor + 1}^{n}
\exp\!\Big(-C_b\frac{m}{(n+1)^b}\Big)
\frac{C}{(n-m+1)^p}
\\
\le{}&
C \sum_{m=\lfloor n/2 \rfloor + 1}^{n}
\exp\!\Big(-C_b\frac{m}{(n+1)^b}\Big)
\\
={}&
C\,
\frac{
    \exp\!\Big(-C_b\frac{\lfloor n/2 \rfloor + 1}{(n+1)^b}\Big)
    \Big(1-\exp\!\big(-C_b\frac{n-\lfloor n/2 \rfloor}{(n+1)^b}\big)\Big)
}{
    1-\exp\!\big(-C_b/(n+1)^b\big)
}
\\
\le{}&
\frac{C\,\exp\!\big(-C_b\frac{\tfrac{n}{2}+1}{(n+1)^b}\big)}
{1-\exp\!\big(-C_b/(n+1)^b\big)}
\\
\le{}&
\frac{3C\,(n+1)^b}{C_b}\,
\exp\!\Big(-C_b\frac{\tfrac{n}{2}+1}{(n+1)^b}\Big),
\end{aligned}
\]
where the last inequality uses \(1-e^{-x}\ge x/3\) applied with
\(x=C_b/(n+1)^b\) (valid for large \(n\)).

Define
\[
A_n
:=
(n+1)^b
\exp\!\Big(-C_b\frac{\tfrac{n}{2}+1}{(n+1)^b}\Big).
\]
To show \(\mathbf{T}_2=o((n+1)^{-r})\) for any \(r>0\), consider
\[
\ln\!\big((n+1)^r A_n\big)
=(r+b)\ln(n+1)-C_b\frac{\tfrac{n}{2}+1}{(n+1)^b}.
\]
Since \(0<b<1\), the second term grows on the order of \(n^{1-b}\) (with factor
\(C_b\)) and dominates the logarithmic term. Hence
\(\ln\!\big((n+1)^r A_n\big)\to -\infty\), so
\((n+1)^r A_n\to 0\). Therefore
\[
\mathbf{T}_2 = o\!\big((n+1)^{-r}\big)\qquad\text{for every }r>0,
\]
and the claim is proved.

\end{proof}

\section{Proof of Almost Sure Convergence (Theorem \ref{thm:almost.sure.conv})}

Before proceeding with the proof of Theorem~\ref{thm:almost.sure.conv}, it is convenient to rewrite the algorithmic updates in the framework of a two--time--scale stochastic approximation scheme. This representation will allow us to clearly separate the dynamics of the local variables from that of the global iterate.

In particular, the local variable corresponding to each worker $w$ evolves on the slower time scale. Its update can be written as
\begin{equation}
y_{n+1}^{(w)} 
= y_n^{(w)} + \beta_n\,\widetilde{h}^{(w)}(n),
\label{eq:slow time scaleA}
\end{equation}
where $\{\beta_n\}$ denotes the associated step--size sequence.

On the other hand, the global iterate $x_n$ evolves on the faster time scale and is updated according to
\begin{equation}
x_{n+1} 
= x_n + \alpha_n\,\widetilde{g}(n),
\label{eq:fast time scaleB}
\end{equation}
where $\{\alpha_n\}$ is the corresponding step--size sequence.

The quantities $\widetilde{h}^{(w)}(n)$ and $\widetilde{g}(n)$ capture the effective increments in the slow and fast time scales, respectively. Using Algorithm~\ref{alg:R-SIGN_Pseudocode}, we now make these expressions explicit under the different settings considered in this work.

\paragraph{Asynchronous setting.}
In the asynchronous regime, at each iteration only one worker $l \in [N]$ and one coordinate $i \in [m]$ are selected. Consequently, the update affects only a single component of a single worker. This leads to
\begin{equation}
\widetilde{h}^{(w)}(n)(j)
= mN \bigl[\,Y_{n+1}^{(w)}(j) - y_n^{(w)}(j)\,\bigr]
\mathbf{1}_{\{j=i\}}
\mathbf{1}_{\{w=l\}},
\label{AFOhwnr}
\end{equation}
where the indicator functions ensure that only the selected coordinate and worker are updated.

The corresponding fast time scale update is given by
\begin{equation}
\widetilde{g}(n)
= Nm\, a_i^{(l)}\,\mathrm{sign}\!\bigl(-y_n^{(l)}(i)\bigr),
\label{asynslowr}
\end{equation}
where $l$ and $i$ denote the randomly selected worker and coordinate, respectively.

\paragraph{Oracle models.}
We now specify how the quantities $Y_{n+1}^{(w)}(j)$ are obtained under different oracle models.

\textbf{First--order case.}
In the first--order setting, the oracle provides (possibly noisy) gradient information, and we have
\begin{equation}
Y_{n+1}^{(w)}(j)
= {a_j^{(w)}}^\top \widetilde{\nabla} f(x_{n- \tau_n^{(w)}}).
\label{AFOoravr}
\end{equation}
Note that the delayed gradient satisfies Assumption \ref{noise-first order}.

\textbf{Zeroth--order case.}
In the zeroth--order setting, gradient information is approximated using finite differences. Specifically,
\begin{equation}
Y_{n+1}^{(w)}(j)
=
\frac{\widehat{f}(x_{n-\tau_n^{(w)}} + \lambda_n a_j^{(w)})
      - \widehat{f}(x_{n-\tau_n^{(w)}} - \lambda_n a_j^{(w)})}
     {2\lambda_n}.
\label{AZOSJDr}
\end{equation}

The oracle evaluations are corrupted by additive noise. In particular,
\[
\widehat{f}(x_{n-\tau_n^{(w)}} + \lambda_n a_i^{(w)})
= f(x_{n-\tau_n^{(w)}} + \lambda_n a_i^{(w)}) + \zeta_{n+1}^{(1)},
\]
\[
\widehat{f}(x_{n-\tau_n^{(w)}} - \lambda_n a_i^{(w)})
= f(x_{n-\tau_n^{(w)}} - \lambda_n a_i^{(w)}) + \zeta_{n+1}^{(2)}.
\]

Here, $\{\zeta_n^{(1)}\}$ and $\{\zeta_n^{(2)}\}$ are martingale difference sequences with respect to the filtration
\[
\mathcal{F}_n := \sigma\!\bigl(x_k, y_k^{(l)}, \zeta_k^{(1)}, \zeta_k^{(2)}, \tau_k^{(l)} \,:\, k \le n,\, l \in [N] \bigr).
\]
Moreover, there exists a constant $\mathrm{K} > 0$ such that
\[
\mathbb{E}\!\left[\|\zeta_{n+1}^{(i)}\|^2 \mid \mathcal{F}_n\right]
\le \mathrm{K},
\qquad i \in \{1,2\},\; n \ge 0.
\]

\subsection{Properties of Drift Term}

\subsubsection{Properties of Drift Term for First-Order Setting}

\begin{lemma}
For any honest worker $w \in [N]$, the following statements hold.

\begin{enumerate}
\item[(a)] The conditional expectation of $\widetilde{h}^{(w)}(n)$ satisfies
\begin{equation*}
\mathbb{E}\!\left[\widetilde{h}^{(w)}(n) \,\big|\, \mathcal{F}_n\right] 
= {A^{(w)}}^\top \nabla f(x_{n-\tau_n^{(w)}}) - y_n^{(w)}.
\end{equation*}

\item[(b)] The second moment admits the bound
\begin{equation*}
\mathbb{E}\!\left[\norm{\widetilde{h}^{(w)}(n)}^2 \,\big|\, \mathcal{F}_n\right]
\;\leq\; 
2 mN \, \norm{ {A^{(w)}}^\top\nabla f(x_{n-\tau_n^{(w)}}) - y_n^{(w)} }^2 
+ 2m N \Bar{A}^2 \sigma^2.
\end{equation*}
\end{enumerate}

\label{lem:approximated gradientf}
\end{lemma}

\begin{proof}
    The proof of (a) is straightforward under the assumption \ref{noise-first order}.

    \emph{proof of (b)}
   \begin{equation}
\begin{aligned}
\mathbb{E}\!\left[
    \big\|\widetilde{h}^{(w)}(n)\big\|^2
    \,\big|\, \mathcal{F}_n
\right]
&= mN \,
\mathbb{E}\!\left[
    \big\|Y_{n+1}^{(w)} - y_n^{(w)}\big\|^2
    \,\big|\, \mathcal{F}_n
\right].
\end{aligned}
\label{eq:afo_second_moment_step1}
\end{equation}

Next, we add and subtract \({A^{(w)}}^\top \nabla f(x_n)\) inside the
norm and apply the inequality \(\|u+v\|^2 \le 2\|u\|^2 + 2\|v\|^2\), which
yields
\begin{equation}
\begin{aligned}
\mathbb{E}\!\left[
    \big\|Y_{n+1}^{(w)} - y_n^{(w)}\big\|^2
    \,\big|\, \mathcal{F}_n
\right]
&\le
2 \big\|{A^{(w)}}^\top \nabla f(x_{n- \tau_n^{(w)}}) - y_n^{(w)}\big\|^2 \\
&\quad
+ 2\,
\mathbb{E}\!\left[
    \big\|
        {A^{(w)}}^\top \widetilde{\nabla} f(x_{n- \tau_n^{(w)}})
        - {A^{(w)}}^\top \nabla f(x_{n- \tau_n^{(w)}})
    \big\|^2
    \,\big|\, \mathcal{F}_n
\right].
\end{aligned}
\label{eq:afo_second_moment_step2}
\end{equation}

In view of Assumption \ref{noise-first order} we obtain
\[
\mathbb{E}\!\left[
    \big\|\widetilde{h}^{(w)}(n)\big\|^2
    \,\big|\, \mathcal{F}_n
\right]
\le
2 mN
\big\|{A^{(w)}}^\top \nabla f(x_n) - y_n^{(w)}\big\|^2
+ 2 mN \bar{A}^2 \sigma^2,
\]
which establishes the desired bound and completes the proof.
\end{proof}

\subsubsection{Properties of Drift Term for Zeroth-Order Setting}

\begin{lemma}
For any honest worker $w \in [N]$, the following statements hold.

\begin{enumerate}
\item[(a)] The conditional expectation of $\widetilde{h}^{(w)}(n)$ satisfies
\begin{equation*}
\mathbb{E}\!\left[\widetilde{h}^{(w)}(n) \,\big|\, \mathcal{F}_n\right] 
= {A^{(w)}}^\top \nabla f(x_{n-\tau_n^{(w)}}) - y_n^{(w)} + \mathrm{B}(n),
\end{equation*}
where the bias term $\mathrm{B}(n)$ obeys the bound
\begin{equation*}
\norm{\mathrm{B}(n)}_\infty \leq L \lambda_n \bar{A}^2.
\end{equation*}

\item[(b)] The second moment admits the bound
\begin{equation*}
\mathbb{E}\!\left[\norm{\widetilde{h}^{(w)}(n)}^2 \,\big|\, \mathcal{F}_n\right]
\;\leq\; 
2 mN \, \norm{ {A^{(w)}}^\top\nabla f(x_{n-\tau_n^{(w)}}) - y_n^{(w)} }^2 
+ \frac{K m}{\lambda_n^2}.
\end{equation*}
\end{enumerate}

\label{lem:approximated gradient}
\end{lemma}

\begin{proof}

\textbf{Proof of (a)}

We begin by introducing the filtration $\mathcal{G}_n$, defined as
\[
\mathcal{G}_n 
= \sigma \Big( \{x_k\}, \{y_k\}, \{\zeta_{k+1}^{(1)}\}, \{\zeta_{k+1}^{(2)}\}, \{\tau_k^{(w)}\} \;\; \big| \;\; 0 \leq k \leq n,\; w \in [N] \Big).
\]
This filtration captures the entire history of the iterates, delays, and noise realizations up to time $n$.

For any honest worker $w$ and any coordinate $j \in [d]$, we now evaluate the conditional expectation of $\widetilde{h}^{(w)}(n)(j)$. We write
\begin{equation}
\begin{split}
\mathbb{E}\!\left[\widetilde{h}^{(w)}(n)(j) \,\big|\, \mathcal{F}_n \right]
&\overset{(a)}{=}
\mathbb{E}\!\left[
    \mathbb{E}\!\left[\widetilde{h}^{(w)}(n)(j)\,\big|\,\mathcal{G}_n\right]
    \,\big|\, \mathcal{F}_n
\right] \\
&\overset{(b)}{=}
\mathbb{E}\!\left[
    Y_{n+1}^{(w)}(j) - y_n^{(w)}(j)
    \,\big|\, \mathcal{F}_n
\right] \\
&\overset{(c)}{=}
\frac{f\!\left(x_{n-\tau_n^{(w)}} + \lambda_n a_j^{(w)}\right)
      - f\!\left(x_{n-\tau_n^{(w)}} - \lambda_n a_j^{(w)}\right)}
     {2\lambda_n}
- y_n^{(w)}(j),
\end{split}
\label{imt}
\end{equation}
Step (a) follows directly from the tower property of conditional expectation.

To establish step (b), we compute the inner conditional expectation explicitly:
\[
\begin{split}
\mathbb{E}\!\left[\widetilde{h}^{(w)}(n)(j)\,\big|\, \mathcal{G}_n\right]
&=
\mathbb{E}\!\left[
    mN \big(Y_{n+1}^{(w)}(j) - y_n^{(w)}(j)\big)
    \mathbf{1}_{\{j=i\}} \mathbf{1}_{\{w=l\}}
    \,\big|\, \mathcal{G}_n
\right] \\
&\overset{(1)}{=}
mN \big(Y_{n+1}^{(w)}(j) - y_n^{(w)}(j)\big)
\mathbb{E}\!\left[
    \mathbf{1}_{\{j=i\}} \mathbf{1}_{\{w=l\}}
    \,\big|\, \mathcal{G}_n
\right] \\
&\overset{(2)}{=}
mN \big(Y_{n+1}^{(w)}(j) - y_n^{(w)}(j)\big)
\mathbb{E}\!\left[
    \mathbf{1}_{\{j=i\}} \mathbf{1}_{\{w=l\}}
\right] \\
&\overset{(3)}{=}
Y_{n+1}^{(w)}(j) - y_n^{(w)}(j).
\end{split}
\]

Here, step (1) uses the fact that $Y_{n+1}^{(w)}(j) - y_n^{(w)}(j)$ is $\mathcal{G}_n$--measurable.  
Step (2) follows from the independence of the random selections $i$ and $l$ from $\mathcal{G}_n$.  
Step (3) uses the identity
\[
\mathbb{E}\!\left[\mathbf{1}_{\{j=i\}} \mathbf{1}_{\{w=l\}}\right]
= \frac{1}{mN}.
\]

Finally, step (c) is obtained by substituting the definition of $Y_{n+1}^{(w)}(j)$ in the zeroth--order oracle model, together with the martingale difference property of the noise. 

Next, we analyze the expression in \eqref{imt} by applying a second--order expansion of the function $f$ along the direction $a_j^{(w)}$ around the delayed iterate $x_{n-\tau_n^{(w)}}$. By the second--order mean value theorem, there exist $\theta_1, \theta_2 \in (0,1)$ such that
\begin{equation}
\begin{split}
&\frac{f\!\left(x_{n-\tau_n^{(w)}} + \lambda_n a_j^{(w)}\right) 
- f\!\left(x_{n-\tau_n^{(w)}} - \lambda_n a_j^{(w)}\right)}{2\lambda_n}
- y_n(j) \\
&= \frac{1}{2\lambda_n} \Big[
    \big( f(x_{n-\tau_n^{(w)}}) 
    + \lambda_n \langle \nabla f(x_{n-\tau_n^{(w)}}), a_j^{(w)} \rangle 
    + \lambda_n^2 \langle a_j^{(w)}, \nabla^2 f(x_{n-\tau_n^{(w)}} + \theta_1 \lambda_n a_j^{(w)}) a_j^{(w)} \rangle \big) \\
&\qquad\qquad
    - \big( f(x_{n-\tau_n^{(w)}}) 
    - \lambda_n \langle \nabla f(x_{n-\tau_n^{(w)}}), a_j^{(w)} \rangle 
    + \lambda_n^2 \langle a_j^{(w)}, \nabla^2 f(x_{n-\tau_n^{(w)}} - \theta_2 \lambda_n a_j^{(w)}) a_j^{(w)} \rangle \big)
\Big] \\
&\qquad - y_n(j).
\end{split}
\label{properties of gradient}
\end{equation}

We now simplify the expression. The zeroth--order terms cancel, and the first--order terms combine as
\[
\frac{1}{2\lambda_n} \cdot 2\lambda_n \langle \nabla f(x_{n-\tau_n^{(w)}}), a_j^{(w)} \rangle
= {a_j^{(w)}}^\top \nabla f(x_{n-\tau_n^{(w)}}).
\]
The remaining second--order terms give rise to a bias term. Therefore,
\[
\frac{f\!\left(x_{n-\tau_n^{(w)}} + \lambda_n a_j^{(w)}\right) 
- f\!\left(x_{n-\tau_n^{(w)}} - \lambda_n a_j^{(w)}\right)}{2\lambda_n}
- y_n(j)
= {a_j^{(w)}}^\top \nabla f(x_{n-\tau_n^{(w)}}) - y_n(j) + B(n),
\]
where the bias term $B(n)$ is defined as
\begin{equation*}
\begin{split}
B(n)
= \frac{\lambda_n}{2} \Big(
\langle a_j^{(w)}, \nabla^2 f(x_{n-\tau_n^{(w)}} + \theta_1 \lambda_n a_j^{(w)}) a_j^{(w)} \rangle
+
\langle a_j^{(w)}, \nabla^2 f(x_{n-\tau_n^{(w)}} - \theta_2 \lambda_n a_j^{(w)}) a_j^{(w)} \rangle
\Big).
\end{split}
\end{equation*}

Finally, under the assumption that $\nabla f$ is Lipschitz continuous with constant $L$, the Hessian is bounded in operator norm by $L$. Hence, we obtain
\[
|B(n)| \;\leq\; L \lambda_n \bar{A}^2,
\]
where $\bar{A}$ is an upper bound on $\|a_j^{(w)}\|$.


\textbf{Proof of (b)}

We begin by expanding the squared norm coordinate-wise:
\begin{equation*}
\mathbb{E}\!\left[\norm{\widetilde{h}^{(w)} (n)}^2 \,\big|\, \mathcal{F}_n \right] 
= \sum_{j=1}^{m} 
\mathbb{E}\!\left[ \big(\widetilde{h}^{(w)} (n)(j)\big)^2 \,\big|\, \mathcal{F}_n \right].
\end{equation*}
Thus, it suffices to obtain an upper bound on each coordinate term 
\(\mathbb{E}\!\left[\big(\widetilde{h}^{(w)} (n)(j)\big)^2 \,\big|\, \mathcal{F}_n \right]\).

Using the tower property of conditional expectation, we write
\begin{equation}
\begin{split}
&\mathbb{E}\!\left[\big(\widetilde{h}^{(w)}(n)(j)\big)^2 \,\big|\, \mathcal{F}_n \right] \\
&= \mathbb{E}\!\left[
    \mathbb{E}\!\left[\big(\widetilde{h}^{(w)}(n)(j)\big)^2 \,\big|\, \mathcal{G}_n \right]
    \,\Big|\, \mathcal{F}_n
\right] \\
&\overset{(a)}{=}
m N \, \mathbb{E}\!\left[
    \big(Y_{n+1}^{(w)}(j) - y_n^{(w)}(j)\big)^2 
    \,\big|\, \mathcal{F}_n
\right],
\end{split}
\label{b prof}
\end{equation}
where step (a) follows from the definition of $\widetilde{h}^{(w)}(n)$ together with the independence of the random coordinate $i \sim U([m])$ and the worker index $l \sim U([N])$ from $\mathcal{G}_n$.

We now bound the squared term. Using the inequality $(a+b)^2 \leq 2(a^2 + b^2)$, we obtain
\begin{equation*}
\begin{split}
&\big(Y_{n+1}^{(w)}(j) - y_n^{(w)}(j)\big)^2 \\
&\leq 
2 \Bigg(
\frac{
f\!\left(x_{n-\tau_n^{(w)}} + \lambda_n a_j^{(w)}\right)
- f\!\left(x_{n-\tau_n^{(w)}} - \lambda_n a_j^{(w)}\right)
}{2\lambda_n}
- y_n^{(w)}(j)
\Bigg)^2 \\
&\quad + \frac{\big(\zeta_{n+1}^{(1)} - \zeta_{n+1}^{(2)}\big)^2}{2\lambda_n^2}.
\end{split}
\end{equation*}

We next bound the expectation of each term separately.

For the first term, using the second--order expansion established in \eqref{properties of gradient}, we obtain
\begin{equation}
\begin{split}
&\mathbb{E}\!\left[
\left(
\frac{
f\!\left(x_{n-\tau_n^{(w)}} + \lambda_n a_j^{(w)}\right)
- f\!\left(x_{n-\tau_n^{(w)}} - \lambda_n a_j^{(w)}\right)
}{2\lambda_n}
- y_n^{(w)}(j)
\right)^2
\,\Big|\, \mathcal{F}_n
\right] \\
&\leq 
\big({a_j^{(w)}}^\top \nabla f(x_{n-\tau_n^{(w)}}) - y_n^{(w)}(j)\big)^2 
+ \mathcal{O}(\lambda_n^2).
\end{split}
\label{first}
\end{equation}

For the second term, under Assumption~\ref{noise-assump}, we have
\begin{equation}
\mathbb{E}\!\left[
\frac{\norm{\zeta_{n+1}^{(1)} - \zeta_{n+1}^{(2)}}^2}{2\lambda_n^2}
\,\Big|\, \mathcal{F}_n
\right]
\leq \frac{K}{\lambda_n^2}.
\label{second}
\end{equation}

Substituting the bounds from \eqref{first} and \eqref{second} into \eqref{b prof}, and summing over all coordinates, completes the proof.
\end{proof}

\begin{remark}
    Note that when the function evaluation noises are same  it can be verified from the Lemma \ref{lem:approximated gradient} that $K = 0$. 
\end{remark}

\subsection{Properties of Drift Term for Global Iterates}
\begin{lemma}
\label{lem:drift_g}
The stochastic drift $\widetilde g(n)$ admits the decomposition
\[
\mathbb{E}\big[\widetilde g(n)\mid\mathcal{F}_n\big] 
= g(x_n,y_n) + b(x_n,y_n),
\]
where
\[
\begin{aligned}
g(x_n,y_n)
&= \sum_{w \in \mathcal{A}} A^{(w)} \,\sign\!\big(-y_n^{(w)}\big)
+ \sum_{w \in \mathcal{A}^c} A^{(w)} \,\sign\!\big(- {A^{(w)}}^\top \nabla f(x_n)\big), \\[4pt]
b(x_n,y_n)
&= \sum_{w \in \mathcal{A}^c} A^{(w)} \,\sign\!\big(-y_n^{(w)}\big)
- \sum_{w \in \mathcal{A}^c} A^{(w)} \,\sign\!\big(- {A^{(w)}}^\top \nabla f(x_n)\big).
\end{aligned}
\]

Moreover, the iterates satisfy the uniform almost--sure bound
\[
\|\widetilde g(n)\| \le m N\,\bar{A}.
\]
\end{lemma}

\begin{proof}
At iteration $n$, a worker--coordinate pair $(l,i)$ is sampled uniformly from 
$[N]\times[m]$. By construction, the update takes the form
\[
\widetilde g(n)
= m N \, a_i^{(l)}\,\mathrm{sign}\!\bigl(-y_n^{(l)}(i)\bigr).
\]

Taking conditional expectation with respect to $\mathcal{F}_n$, and using the independence of $(l,i)$ from $\mathcal{F}_n$, we obtain
\[
\begin{aligned}
\mathbb{E}\big[\widetilde g(n)\mid\mathcal{F}_n\big]
&= \sum_{w=1}^N \sum_{j=1}^m 
a_j^{(w)} \,\mathrm{sign}\!\bigl(-y_n^{(w)}(j)\bigr).
\end{aligned}
\]

We now separate the contribution of honest and adversarial workers. Adding and subtracting the term
\[
\sum_{w\in\mathcal{A}^c} \sum_{j=1}^m 
a_j^{(w)} \,\mathrm{sign}\!\big(-{a_j^{(w)}}^\top\nabla f(x_n)\big),
\]
we obtain
\[
\mathbb{E}\big[\widetilde g(n)\mid\mathcal{F}_n\big]
= g(x_n,y_n) + b(x_n,y_n),
\]
where $g(x_n,y_n)$ and $b(x_n,y_n)$ are defined as above.

Finally, since $\|a_j^{(w)}\|\le \bar A$ for all $j$ and $w$, we have the pointwise bound
\[
\|\widetilde g(n)\|
= m N \,\|a_i^{(l)}\|
\le m N \bar A,
\]
which completes the proof.
\end{proof}

\subsection{A Two Time Scale Stochastic Approximation Scheme}
We first record several useful properties of the stochastic drift terms
$\widetilde{h}^{(w)}(n)$ and $\widetilde{g}(n)$ that will be used repeatedly in the subsequent analysis.

The conditional expectation of $\widetilde{h}^{(w)}(n)$ satisfies
\begin{equation}
\mathbb{E}\!\left[
\widetilde{h}^{(w)}(n)\mid \mathcal{F}_n
\right]
=
{A^{(w)}}^\top \nabla f(x_{n-\tau_n^{(w)}})
-
y_n^{(w)}
+
B_n,
\qquad
\|B_n\|_\ast \le K_1 \lambda_n ,
\label{hwn}
\end{equation}
where the term $B_n$ represents the bias introduced by the gradient approximation.

Further, the conditional second moment of
$\widetilde{h}^{(w)}(n)$ admits the bound
\begin{equation}
\mathbb{E}\!\left[
\|\widetilde{h}^{(w)}(n)\|^2
\mid \mathcal{F}_n
\right]
\le
K_2
\left\|
{A^{(w)}}^\top \nabla f(x_{n-\tau_n^{(w)}})
-
y_n^{(w)}
\right\|^2
+
\frac{K_3}{\lambda_n^2}
+
K_4 .
\label{C2M}
\end{equation}

Similarly, the stochastic drift term $\widetilde g(n)$ satisfies
\begin{equation}
\mathbb{E}\!\left[
\widetilde g(n)
\mid
\mathcal{F}_n
\right]
=
g(x_n,y_n)
+
b(x_n,y_n),
\label{Tgn}
\end{equation}
where $b(x_n,y_n)$ denotes the corresponding bias term. In addition, we have the uniform bound
\begin{equation}
\|\widetilde g(n)\|
\le
K_5 .
\label{rysh}
\end{equation}

The constants $K_1,K_2,K_3,K_4,$ and $K_5$ are given in Lemma \ref{lem:approximated gradientf}, 
Lemma~\ref{lem:approximated gradient} and
Lemma~\ref{lem:drift_g}.

\subsection{Proof of Almost Sure Boundedness}
In this subsection, we prove the following Theorem. 
\begin{theorem}[Boundedness of the Iterate] Suppose that Assumptions~\ref{L-Smooth}, \ref{noise-assump}, \ref{Step-sizeA}, hold. Assume further that the stochastic drift terms
satisfy conditions \eqref{hwn}, \eqref{C2M}, \eqref{Tgn}, and
\eqref{rysh}. 
 Then  the following hold:
    \begin{enumerate}
      \item The iterate sequence $\{x_n \}$   is almost surely bounded:
\begin{equation*}
            \mathbb{P}(\sup\limits_{n \geq 1} \norm{x_n} < \infty) = 1. 
        \end{equation*}
 \item The iterate sequence $\{y_n^{(w)} \}$  for each honest worker is also almost surely bounded:

        \begin{equation*}
            \mathbb{P}(\sup\limits_{n \geq 1} \norm{y_n^{(w)}} < \infty) = 1. 
        \end{equation*}
    \end{enumerate}
    \label{thm:almst boundeness}
\end{theorem}
The proof strategy for Theorem~\ref{thm:almst boundeness} proceeds in two steps.

\textbf{Proof Strategy}

\begin{enumerate}
    \item \textbf{Step 1 (Fast-timescale tracking):} We show that the fast iterate $y_n^{(w)}$ tracks the directional gradient ${A^{(w)}}^\top \nabla f(x_n)$ almost surely. This is achieved by establishing that the squared tracking error $\norm{e_n^{(w)}}^2$ satisfies an almost supermartingale recursion (Theorem \ref{Robbins Siegmund}), as detailed in Theorem \ref{Corol:grdioent trachking}. The main non-triviality lies in controlling the bias introduced by the time-varying parameter $x_n$ and the asynchronous delays. We bound the iterate movement via $\norm{x_{n+1} - x_n} \leq \mathcal{O}(\alpha_n)$, and handle the delay-induced error using Cauchy-Schwarz:
    \begin{equation*}
        \norm{x_{n-\tau_n^{(w)}} - x_n}^2 \le K_5^2 \, \tau_n^{(w)} \sum_{k=n-\tau_n^{(w)}}^{n-1} \alpha_k^2.
    \end{equation*}
    Lemma \ref{lem:delay-summable} then ensures this delay perturbation is summable by proving $\sum_{n \ge 1} \tau_n^{(w)} \sum_{k=n-\tau_n^{(w)}}^{n-1} \alpha_k^2 < \infty$ a.s.

    \item \textbf{Step 2 : Stability of $\mathbf{x_n}$} We show that the objective suboptimality $f(x_n) - \inf f(x)$ also satisfies an almost supermartingale recursion (Lemma \ref{descentineq}). To successfully apply Theorem \ref{Robbins Siegmund} to the objective sequence, we must guarantee that $\sum_{n \ge 1} \alpha_n \norm{e_n^{(w)}}_1 < \infty$ almost surely, a crucial summability condition formally proven in Theorem \ref{Corol:grdioent trachking}. Consequently, we see that $f(x_n)$
converges almost surely.
Finally, since the objective function $f$ is coercive, the almost sure convergence of $f(x_n)$ implies that the sequence $\{x_n\}$ remains almost surely bounded.
\end{enumerate}

Before proceeding we need the following Proposition.
\begin{proposition}
\label{prop:enbound}
Fix a worker $w \in \mathcal{A}^c$ and define
\[
e_n^{(w)} := y_n^{(w)} - {A^{(w)}}^\top \nabla f(x_n).
\]
Then
\begin{equation}
\begin{split}
\mathbb{E}\!\left[\|e_{n+1}^{(w)}\|^2 \mid \mathcal{F}_n\right]
\le
\big(1 - A_n\big)\|e_n^{(w)}\|^2
+ B_n \|e_n^{(w)}\|
+ C_n + D_n,
\end{split}
\label{18eq}
\end{equation}
where
\[
A_n := 2\beta_n - c_2 \beta_n - 4 K_2 \beta_n^2 - \beta_n^2, \qquad c_2>0,
\] 

\[
B_n := 2 \bar{A} L K_5 \alpha_n,
\]
\[
C_n :=
\frac{2K_3 \beta_n^2}{\lambda_n^2}
+ \frac{\beta_n}{c_2} K_1^2 \lambda_n^2
+ 2\alpha_n^2 L^2 \bar{A}^2 K_5^2
+ 2K_4 \beta_n^2,
\]
and
\[
D_n :=
4 K_2 L^2 \bar{A}^2 K_5^2 \beta_n^2 \tau_n^{(w)}
\sum_{k=n-\tau_n^{(w)}}^{n-1} \alpha_k^2
+ L^2 \bar{A}^2 K_5^2 \tau_n^{(w)}
\sum_{k=n-\tau_n^{(w)}}^{n-1} \alpha_k^2.
\]
\end{proposition}
\begin{proof}
 From \eqref{eq:slow time scaleA}, the iterate update for each honest worker $w$ is  
\[
    y_{n+1}^{(w)} = y_n^{(w)} + \beta_n \Tilde{h}^{(w)}(n).
\]  
Define the error vector for each worker $w$ as  
\[
    e_{n+1}^{(w)} = y_{n+1}^{(w)} - {A^{(w)}}^\top \nabla f(x_{n+1}).
\]  
Then we have  
\begin{equation}
\begin{split}
\|e_{n+1}^{(w)}\|^2
&= \Big\|
y_n^{(w)} + \beta_n \widetilde{h}^{(w)}(n) - {A^{(w)}}^\top \nabla f(x_{n+1})
\Big\|^2 \\
&= \Big\|
y_n^{(w)} + \beta_n \widetilde{h}^{(w)}(n) - {A^{(w)}}^\top \nabla f(x_n)
+ {A^{(w)}}^\top \nabla f(x_n) - {A^{(w)}}^\top \nabla f(x_{n+1})
\Big\|^2 \\
&=
\underbrace{\big\|
y_n^{(w)} + \beta_n \widetilde{h}^{(w)}(n) - {A^{(w)}}^\top \nabla f(x_n)
\big\|^2}_{(a)} \\
&\quad +
\underbrace{\big\|
{A^{(w)}}^\top \nabla f(x_n) - {A^{(w)}}^\top \nabla f(x_{n+1})
\big\|^2}_{(b)} \\
&\quad +
\underbrace{2 \Big(
y_n^{(w)} + \beta_n \widetilde{h}^{(w)}(n) - {A^{(w)}}^\top \nabla f(x_n)
\Big)^\top 
{A^{(w)}}^\top\big(
\nabla f(x_n) - \nabla f(x_{n+1})
\big)}_{(c)}.
\end{split}
\label{abc1}
\end{equation}

Taking conditional expectation with respect to $\mathcal{F}_n$ on both sides of \eqref{abc1}, we bound $(a)$, $(b)$, and $(c)$ separately.

\textbf{Bound on the term (a)}

Consider term~(a). Taking conditional expectation with respect to
$\mathcal{F}_n$ yields
\begin{equation*}
\begin{split}
    &\mathbb{E}\!\left[
    \big\|
    y_n^{(w)} + \beta_n \widetilde{h}^{(w)}(n)
    -
    {A^{(w)}}^\top \nabla f(x_n)
    \big\|^2
    \,\middle|\,
    \mathcal{F}_n
    \right]
    \\[0.3em]
    \le{}&
    \|e_n^{(w)}\|^2
    +
    \underbrace{\beta_n^2
    \mathbb{E}\!\left[
    \|\widetilde{h}^{(w)}(n)\|^2
    \mid
    \mathcal{F}_n
    \right]}_{(d)}
    +
    \underbrace{2\beta_n (e_n^{(w)})^\top
    \mathbb{E}\!\left[
    \widetilde{h}^{(w)}(n)
    \mid
    \mathcal{F}_n
    \right]}_{(e)}.
\end{split}
\end{equation*}
\textbf{Bound on the term (e)}
\begin{equation*}
\begin{split}
& 2\beta_n (e_n^{(w)})^\top
\mathbb{E}\!\left[
\widetilde{h}^{(w)}(n)
\mid \mathcal{F}_n
\right] \\
&= 2 \beta_n (e_n^{(w)})^\top
\big(
{A^{(w)}}^\top \nabla f(x_{n-\tau_n^{(w)}})
- y_n^{(w)}
+ B_n
\big) \\
&\overset{(a)}{=}
2 \beta_n (e_n^{(w)})^\top
\big(
- e_n^{(w)}
+ {A^{(w)}}^\top \nabla f(x_{n-\tau_n^{(w)}})
- {A^{(w)}}^\top \nabla f(x_n)
+ B_n
\big) \\
&\overset{(b)}{\le}
- 2 \beta_n \norm{e_n^{(w)}}^2
+ \underbrace{
2 \beta_n (e_n^{(w)})^\top
\big(
{A^{(w)}}^\top \nabla f(x_{n-\tau_n^{(w)}})
- {A^{(w)}}^\top \nabla f(x_n)
\big)
}_{(f)} \\
&\quad + c_2 \beta_n \norm{e_n^{(w)}}^2
+ \frac{\beta_n}{c_2} K_1^2 \lambda_n^2.
\end{split}
\end{equation*}

Step (a) uses
\[
e_n^{(w)} = y_n^{(w)} - {A^{(w)}}^\top \nabla f(x_n).
\]
Step (b) follows from
\[
\langle a,b\rangle \le \frac{c}{2}\|a\|^2 + \frac{1}{2c}\|b\|^2, \qquad c>0,
\]
applied to $(e_n^{(w)})^\top B_n$ with $c=c_2$ and $\|B_n\|\le K_1 \lambda_n$.

\textbf{Bound on the term (f)}

Note that
\begin{equation}
\begin{split}
& 2 \beta_n (e_n^{(w)})^\top
\big(
{A^{(w)}}^\top \nabla f(x_{n-\tau_n^{(w)}})
- {A^{(w)}}^\top \nabla f(x_n)
\big) \\
&\le 
\beta_n^2 \norm{e_n^{(w)}}^2
+ \norm{
{A^{(w)}}^\top \nabla f(x_{n-\tau_n^{(w)}})
- {A^{(w)}}^\top \nabla f(x_n)
}^2 \\
&\le 
\beta_n^2 \norm{e_n^{(w)}}^2
+ L^2 \bar{A}^2 \norm{x_{n-\tau_n^{(w)}} - x_n}^2.
\end{split}
\label{eq:f1}
\end{equation}

Further, in view of \eqref{eq:fast time scaleB} we obtain
\begin{equation*}
\begin{split}
\norm{x_{n-\tau_n^{(w)}} - x_n}
&\le \sum_{k=n-\tau_n^{(w)}}^{n-1} \alpha_k \norm{\widetilde{g}(k)} \le K_5 \sum_{k=n-\tau_n^{(w)}}^{n-1} \alpha_k.
\end{split}
\end{equation*}

Using Cauchy--Schwarz $(\sum\limits_{k \geq 1} u_k v_k)^2 \leq \sum\limits_{k \geq} u_k^2  \sum\limits_{k \geq} u_k^2$, we obtain
\begin{equation*}
\begin{split}
\norm{x_{n-\tau_n^{(w)}} - x_n}^2
&\le K_5^2
\left( \sum_{k=n-\tau_n^{(w)}}^{n-1} \alpha_k \right)^2 \le K_5^2 \, \tau_n^{(w)}
\sum_{k=n-\tau_n^{(w)}}^{n-1} \alpha_k^2.
\end{split}
\end{equation*}

\textbf{Bound on the term (d)}

\begin{equation*}
\begin{split}
& \beta_n^2
\mathbb{E}\!\left[
\|\widetilde{h}^{(w)}(n)\|^2
\mid \mathcal{F}_n
\right] \\
&\le 
\beta_n^2 \Big(
K_2 \norm{
{A^{(w)}}^\top \nabla f(x_{n-\tau_n^{(w)}})
- y_n^{(w)}
}^2
+ \frac{K_3}{\lambda_n^2}
+ K_4
\Big) \\
&\le 
\beta_n^2 \Big(
2 K_2 \norm{
{A^{(w)}}^\top \nabla f(x_{n-\tau_n^{(w)}})
- {A^{(w)}}^\top \nabla f(x_n)
}^2
+ 2 K_2 \norm{
{A^{(w)}}^\top \nabla f(x_n) - y_n^{(w)}
}^2 \\
&\qquad\qquad
+ \frac{K_3}{\lambda_n^2}
+ K_4
\Big) \\
&\le 
\beta_n^2 \Big(
2 K_2 L^2 \bar{A}^2 
\norm{x_{n-\tau_n^{(w)}} - x_n}^2
+ 2 K_2 \norm{e_n^{(w)}}^2
+ \frac{K_3}{\lambda_n^2}
+ K_4
\Big) \\
&\le 
\beta_n^2 \Big(
2 K_2 L^2 \bar{A}^2 K_5^2 \tau_n^{(w)}
\sum_{k=n-\tau_n^{(w)}}^{n-1} \alpha_k^2
+ 2 K_2 \norm{e_n^{(w)}}^2
+ \frac{K_3}{\lambda_n^2}
+ K_4
\Big).
\end{split}
\end{equation*}

\textbf{Bound on the term (b)}

\begin{equation}
\begin{split}
& \| {A^{(w)}}^\top \nabla f(x_n) - {A^{(w)}}^\top \nabla f(x_{n+1})\|^2 \\
&\overset{(a)}{\leq} \norm{A^{(w)}}^2 \, \norm{\nabla f(x_n) - \nabla f(x_{n+1})}^2 \\
&\overset{(b)}{\leq} \bar{A}^2 L^2 \, \norm{x_{n+1}-x_n}^2 \\
&\overset{(c)}{\leq} \bar{A}^2 L^2 \, \alpha_n^2 K_5^2.
\end{split}
\label{ssmothprop}
\end{equation}

(a) uses $\|A^\top v\|\le \|A\|\|v\|$. 
(b) uses $L$-smoothness. 
(c) follows from the explicit form \eqref{eq:fast time scaleB}.

\textbf{Bound on the term (c)}

\begin{equation*}
\begin{split}
& 2\big(y_n^{(w)} + \beta_n \widetilde{h}^{(w)}(n) - {A^{(w)}}^\top \nabla f(x_n)\big)^\top
{A^{(w)}}^\top\big(\nabla f(x_n) - \nabla f(x_{n+1})\big) \\
&= 2\big(e_n^{(w)} + \beta_n \widetilde{h}^{(w)}(n)\big)^\top
{A^{(w)}}^\top\big(\nabla f(x_n) - \nabla f(x_{n+1})\big) \\
&\le 
2 \norm{e_n^{(w)}} \,
\norm{{A^{(w)}}^\top(\nabla f(x_n)-\nabla f(x_{n+1}))} \\
&\quad +
\beta_n^2 \norm{\widetilde{h}^{(w)}(n)}^2
+ \norm{{A^{(w)}}^\top(\nabla f(x_n)-\nabla f(x_{n+1}))}^2 \\
&\le 
2 \bar{A} L \norm{e_n^{(w)}} \norm{x_{n+1}-x_n}
+ \beta_n^2 \norm{\widetilde{h}^{(w)}(n)}^2
+ \bar{A}^2 L^2 \norm{x_{n+1}-x_n}^2 \\
&\le 
2 \bar{A} L K_5 \alpha_n \norm{e_n^{(w)}}
+ \beta_n^2 \norm{\widetilde{h}^{(w)}(n)}^2
+ \bar{A}^2 L^2 K_5^2 \alpha_n^2.
\end{split}
\end{equation*}
The first inequality is in view of \[ 2 a b  \leq  a^2 + b^2,
\] and the Cauchy-Schwarz $\inprod{a}{b} \leq \norm{a}\norm{b}$. The last inequality follows from \eqref{ssmothprop}. Combining the bounds for terms~(a), (b), and (c) in~\eqref{abc1} we obtain the result. 
\end{proof}

In the next theorem, we show that the stochastic error sequence 
$
\{\|e_n^{(w)}\|^2\}
$
satisfies an almost supermartingale recursion in the sense of the Robbins--Siegmund theorem. This property allows us to apply the Robbins--Siegmund theorem directly and conclude that
\[
e_n^{(w)} \to 0
\qquad \text{a.s.}
\]
\begin{theorem}
    \begin{enumerate}
        \item 
For each honest worker $w$, the sequence $\{y_n^{(w)}\}$ almost surely tracks the true gradient in the sense that
\[
    \lim_{n \to \infty} \|y_n^{(w)} - {A^{(w)}}^\top \nabla f(x_n)\| = 0 \qquad \text{a.s.}
\]
\item For each honest worker $w$, the sequence $\{\alpha_n \|e_n^{(w)}\|\}$ is summable almost surely:
\[
    \sum_{n \geq 1} \alpha_n \|e_n^{(w)}\| < \infty \qquad \text{a.s.}
\]
    \end{enumerate}
    \label{Corol:grdioent trachking}
\end{theorem}
\begin{proof}
To prove the theorem, we set $c_2 = 1$ and rewrite
Proposition~\ref{prop:enbound} in the form
\begin{equation}
\begin{split}
    \mathbb{E}\!\left[\|e_{n+1}^{(w)}\|^2 \mid \mathcal{F}_n\right]
    \le{}&
    \big(1 + K_2 \beta_n^2\big)\|e_n^{(w)}\|^2
    \\ &\quad
    +
    \underbrace{
    \Big(
    -\beta_n \|e_n^{(w)}\|^2
    +
    2\bar{A} L K_5 \alpha_n \|e_n^{(w)}\|
    \Big)
    }_{\mathbf{T}_4}
    + C_n + D_n .
\end{split}
\label{23er}
\end{equation}

We now analyze the term $\mathbf{T}_4$. The first term is negative and provides a stabilizing effect, whereas the second term acts as a perturbation. To make this structure explicit, we complete the square:
\begin{equation}
\begin{split}
    &-\beta_n \|e_n^{(w)}\|^2
    +
    2\bar{A} L K_5 \alpha_n \|e_n^{(w)}|
    \\[0.2em]
    ={}&
    -\beta_n
    \Bigg(
    \|e_n^{(w)}\|^2
    -
    \frac{2\bar{A} L K_5 \alpha_n}{\beta_n}
    \|e_n^{(w)}\|
    +
    \frac{\bar{A}^2 L^2 K_5^2 \alpha_n^2}{\beta_n^2}
    \Bigg)
    +
    \frac{\bar{A}^2 L^2 K_5^2 \alpha_n^2}{\beta_n}
    \\[0.2em]
    ={}&
    -\beta_n
    \Bigg(
    \|e_n^{(w)}\|
    -
    \frac{\bar{A} L K_5 \alpha_n}{\beta_n}
    \Bigg)^2
    +
    \frac{\bar{A}^2 L^2 K_5^2 \alpha_n^2}{\beta_n}.
\end{split}
\end{equation}

Substituting the above expression into~\eqref{23er}, we obtain
\begin{equation}
\begin{split}
    \mathbb{E}\!\left[\|e_{n+1}^{(w)}\|^2 \mid \mathcal{F}_n\right]
    \le{}&
    \big(1 + K_2 \beta_n^2 \big)\|e_n^{(w)}\|^2
    -
    \beta_n
    \Bigg(
    \|e_n^{(w)}\|
    -
    \frac{\bar{A} L K_5 \alpha_n}{\beta_n}
    \Bigg)^2
    \\ &\quad
    +
    C_n
    +
    \frac{\bar{A}^2 L^2 K_5^2 \alpha_n^2}{\beta_n}
    +
    D_n.
\end{split}
\label{23er1}
\end{equation}

Next, observe that Assumption~\ref{Step-sizeA} implies
\[
\sum_{n \ge 1} \beta_n^2 < \infty,
\]
and
\[
\sum_{n \ge 1}
\left(
C_n
+
\frac{\bar{A}^2 L^2 K_5^2 \alpha_n^2}{\beta_n}
\right)
< \infty .
\]

Therefore, in order to verify the hypotheses of
Theorem~\ref{Robbins Siegmund}, it remains to establish that
\[
\sum_{n \ge 1} D_n < \infty
\qquad \text{a.s.}
\]

By the definition of $D_n$, this reduces to proving that
\begin{equation}
\sum\limits_{n \geq 1}\tau_n^{(w)}
\sum_{k=n-\tau_n^{(w)}}^{n-1}
\alpha_k^2
< \infty \quad \text{a.s.} .
\label{Delay}
\end{equation}

However, \eqref{Delay} holds in view of Lemma \ref{lem:delay-summable}.  

Therefore, by applying Theorem~\ref{Robbins Siegmund}, we obtain:
\begin{itemize}
    \item The sequence $\{\|e_n^{(w)}\|^2\}$ converges almost surely to a finite
    random variable.
    \item Moreover,
    \begin{equation}
        \label{enw}
        \sum_{n \ge 1}
        \beta_n
        \left(
        \|e_n^{(w)}\|
        -
        \frac{L \alpha_n K_5 \bar{A}}{\beta_n}
        \right)^2
        < \infty
        \qquad \text{a.s.}
    \end{equation}
\end{itemize}

\textbf{Proof of Statement 1.}

Since $\sum_{n \ge 1} \beta_n = \infty$, the summability condition in~\eqref{enw}
implies that
\[
    \liminf_{n \to \infty}
    \left(
    \|e_n^{(w)}\|
    -
    \frac{L \alpha_n K_5 \bar{A}}{\beta_n}
    \right)^2
    = 0,
\]
or equivalently,
\[
    \liminf_{n \to \infty} \|e_n^{(w)}\| = 0.
\]
Together with the almost sure convergence of $\|e_n^{(w)}\|$ to a finite random
variable, this yields
\[
    \lim_{n \to \infty} \|e_n^{(w)}\| = 0
    \qquad \text{almost surely}.
\]

\textbf{Proof of Statement -2.}

Observe that
\[
    0 \le
    \beta_n
    \Big(
    \|e_n^{(w)}\|
    -
    \tfrac{L\alpha_n K_5 \bar{A}}{\beta_n}
    \Big)^2.
\]
By the Young--Fenchel inequality,
\[
    \alpha_n \|e_n^{(w)}\|
    \le
    \frac{1}{2}\beta_n \|e_n^{(w)}\|^2
    +
    \frac{\alpha_n^2}{2\beta_n}.
\]
Moreover, using the inequality $(a+b)^2 \le 2(a^2+b^2)$, we obtain
\begin{equation*}
\begin{split}
    \beta_n \|e_n^{(w)}\|^2
    \le{}&
    2\beta_n
    \Bigg[
    \Big(
    \|e_n^{(w)}\|
    -
    \frac{L\alpha_n K_5 \bar{A}}{\beta_n}
    \Big)^2
    +
    \frac{L^2 K_5^2 \bar{A}^2 \alpha_n^2}{\beta_n^2}
    \Bigg].
\end{split}
\end{equation*}
Consequently,
\[
\begin{aligned}
    \sum_{n \ge 1} \alpha_n \|e_n^{(w)}\|
    \le{}&
    \sum_{n \ge 1}
    \frac{1}{2}\beta_n \|e_n^{(w)}\|^2
    +
    \frac{\alpha_n^2}{2\beta_n}
    \\[0.2em]
    \le{}&
    \sum_{n \ge 1}
    \beta_n
    \Big(
    \|e_n^{(w)}\|
    -
    \tfrac{L\alpha_n K_5 \bar{A}}{\beta_n}
    \Big)^2
    +
    \sum_{n \ge 1}
    \frac{\alpha_n^2}{2\beta_n}
    \big(
    1
    +
    2L^2 K_5^2 \bar{A}^2
    \big).
\end{aligned}
\]
By Assumption~\ref{Step-sizeA}, the right-hand side is finite almost surely, which
establishes the claim.
\end{proof}

\begin{lemma}
\label{lem:delay-summable}
\[
\mathbb{E}[\sum_{n \ge 1}
\tau_n^{(w)}
\sum_{k=n-\tau_n^{(w)}}^{n-1}
\alpha_k^2] \leq \mathbb{E}[\tau^2]
\sum_{k \ge 1}
\alpha_k^2
\]
\end{lemma}

\begin{proof}
Consider the nonnegative random series
\[
\sum_{n \ge 1}
\tau_n^{(w)}
\sum_{k=n-\tau_n^{(w)}}^{n-1}
\alpha_k^2  .
\]
Since all terms are nonnegative, Tonelli's theorem allows us to interchange the order of expectation and summation. Hence,
\begin{equation*}
\begin{split}
&\mathbb{E}\!\left[
\sum_{n \ge 1}
\tau_n^{(w)}
\sum_{k=n-\tau_n^{(w)}}^{n-1}
\alpha_k^2
\right]
\\
={}&
\mathbb{E}\!\left[
\sum_{n \ge 1}
\tau_n^{(w)}
\sum_{k=1}^{\infty}
\alpha_k^2
\mathbf{1}_{\{n-\tau_n^{(w)} \le k \le n-1\}}
\right]
\\
={}&
\sum_{k \ge 1}
\alpha_k^2
\sum_{n=k+1}^{\infty}
\mathbb{E}\!\left[
\tau_n^{(w)}
\mathbf{1}_{\{\tau_n^{(w)} \ge n-k\}}
\right].
\end{split}
\end{equation*}

Now introduce the change of variables
\[
m = n-k.
\]
Then
\begin{equation*}
\begin{split}
&\sum_{k \ge 1}
\alpha_k^2
\sum_{n=k+1}^{\infty}
\mathbb{E}\!\left[
\tau_n^{(w)}
\mathbf{1}_{\{\tau_n^{(w)} \ge n-k\}}
\right]
\\
={}&
\sum_{k \ge 1}
\alpha_k^2
\sum_{m=1}^{\infty}
\mathbb{E}\!\left[
\tau_{k+m}^{(w)}
\mathbf{1}_{\{\tau_{k+m}^{(w)} \ge m\}}
\right].
\end{split}
\end{equation*}

By Assumption~\ref{ass:delay},
\[
\tau_n^{(w)} \le \tau
\qquad \text{a.s.}
\]
for all $n$, where $\tau$ is an integer-valued random variable satisfying
\[
\mathbb{E}[\tau^2] < \infty.
\]
Therefore,
\begin{equation*}
\begin{split}
&\sum_{k \ge 1}
\alpha_k^2
\sum_{m=1}^{\infty}
\mathbb{E}\!\left[
\tau_{k+m}^{(w)}
\mathbf{1}_{\{\tau_{k+m}^{(w)} \ge m\}}
\right]
\\
\le{}&
\sum_{k \ge 1}
\alpha_k^2
\,
\mathbb{E}\!\left[
\tau
\sum_{m=1}^{\infty}
\mathbf{1}_{\{\tau \ge m\}}
\right].
\end{split}
\end{equation*}

Since $\tau$ is integer-valued, we have
\[
\sum_{m=1}^{\infty}
\mathbf{1}_{\{\tau \ge m\}}
=
\tau.
\]
Consequently,
\[
\mathbb{E}\!\left[
\sum_{n \ge 1}
\tau_n^{(w)}
\sum_{k=n-\tau_n^{(w)}}^{n-1}
\alpha_k^2
\right]
\le
\mathbb{E}[\tau^2]
\sum_{k \ge 1}
\alpha_k^2
<
\infty .
\]

Since the random variable
\[
\sum_{n \ge 1}
\tau_n^{(w)}
\sum_{k=n-\tau_n^{(w)}}^{n-1}
\alpha_k^2
\]
is nonnegative and has finite expectation, it follows that it is finite almost surely. This completes the proof.
\end{proof}
To complete the proof of almost sure boundedness of the iterates $\{x_n\}$, it remains to analyze the behavior of the objective sequence $\{f(x_n)\}$. The next lemma shows that $\{f(x_n)\}$ satisfies an almost supermartingale-type recursion.

\begin{lemma}
\label{descentineq}
Under Assumptions~\ref{L-Smooth}, \ref{noise-assump}, and~\ref{fawzi-condn}, the following descent inequality holds:
\[
\begin{aligned}
    \mathbb{E}[f(x_{n+1}) \mid \mathcal{F}_n]
    \leq f(x_n)
    &- \alpha_n \eta \|\nabla f(x_n)\|_1
    + 2\alpha_n |\mathcal{A}^c| \|e_n^{(w)}\|_1
    + \frac{L}{2}\alpha_n^2 K_5^2,
\end{aligned}
\]
where $\eta > 0$ is a constant.
\end{lemma}

\begin{proof}
Since $f$ is $L$--smooth by Assumption~\ref{L-Smooth}, we have the standard descent inequality
\begin{equation}
    f(x_{n+1})
    \le
    f(x_n)
    +
    \nabla f(x_n)^\top (x_{n+1}-x_n)
    +
    \frac{L}{2}\|x_{n+1}-x_n\|^2 .
    \label{smoothness-1}
\end{equation}

Recall that the global iterate evolves according to
\[
    x_{n+1}
    =
    x_n + \alpha_n \widetilde g(n).
\]
Substituting this update into~\eqref{smoothness-1} gives
\begin{equation*}
\begin{split}
    f(x_{n+1})
    \le{}&
    f(x_n)
    +
    \alpha_n
    \nabla f(x_n)^\top \widetilde g(n)
    +
    \frac{L}{2}\alpha_n^2
    \|\widetilde g(n)\|^2 .
\end{split}
\end{equation*}

Taking conditional expectation with respect to $\mathcal{F}_n$, we obtain
\begin{equation}
\begin{split}
    \mathbb{E}[f(x_{n+1}) \mid \mathcal{F}_n]
    \le{}&
    f(x_n)
    +
    \alpha_n
    \nabla f(x_n)^\top
    \mathbb{E}[\widetilde g(n)\mid \mathcal{F}_n]
    \\
    &\quad
    +
    \frac{L}{2}\alpha_n^2
    \mathbb{E}[\|\widetilde g(n)\|^2 \mid \mathcal{F}_n].
\end{split}
\label{cond11}
\end{equation}

Next, recall that from Lemma \ref{lem:drift_g}
\begin{equation}
    \mathbb{E}[\widetilde g(n)\mid\mathcal{F}_n]
    =
    g(x_n,y_n)
    +
    b(x_n,y_n).
    \label{rmmm}
\end{equation}
Substituting~\eqref{rmmm} into~\eqref{cond11}, we obtain
\begin{equation*}
\begin{split}
    \mathbb{E}[f(x_{n+1}) \mid \mathcal{F}_n]
    \le{}&
    f(x_n)
    +
    \alpha_n
    \nabla f(x_n)^\top g(x_n,y_n)
    \\
    &\quad
    +
    \alpha_n
    \nabla f(x_n)^\top b(x_n,y_n)
    +
    \frac{L}{2}\alpha_n^2 K_5^2 .
\end{split}
\end{equation*}

We now estimate the two drift terms separately.

\medskip

\noindent
\textbf{Step 1: Bound on $\nabla f(x_n)^\top g(x_n,y_n)$.}

Recall from~ Lemma \ref{lem:drift_g} that
\[
    g(x_n,y_n)
    =
    \sum_{w \in \mathcal{A}}
    A^{(w)} \sign(-y_n^{(w)})
    +
    \sum_{w \in \mathcal{A}^c}
    A^{(w)}
    \sign\!\big(-{A^{(w)}}^\top \nabla f(x_n)\big).
\]
Taking inner product with $\nabla f(x_n)$ gives
\begin{equation*}
\begin{split}
    \nabla f(x_n)^\top g(x_n,y_n)
    ={}&
    \sum_{w \in \mathcal{A}}
    \nabla f(x_n)^\top
    A^{(w)}
    \sign(-y_n^{(w)})
    \\
    &\quad
    +
    \sum_{w \in \mathcal{A}^c}
    \nabla f(x_n)^\top
    A^{(w)}
    \sign\!\big(-{A^{(w)}}^\top \nabla f(x_n)\big).
\end{split}
\end{equation*}

Using the identity
\[
a^\top \sign(-a) = -\|a\|_1,
\]
we obtain
\[
\begin{aligned}
    \nabla f(x_n)^\top g(x_n,y_n)
    \le{}&
    \sum_{w \in \mathcal{A}}
    \|{A^{(w)}}^\top \nabla f(x_n)\|_1
    -
    \sum_{w \in \mathcal{A}^c}
    \|{A^{(w)}}^\top \nabla f(x_n)\|_1 .
\end{aligned}
\]

Finally, Assumption~\ref{fawzi-condn} implies
\[
    \nabla f(x_n)^\top g(x_n,y_n)
    \le
    -\eta \|\nabla f(x_n)\|_1 \quad \eta > 0 \quad \text{(see Lemma 2; \cite{ganesh2023online} for the detailed proof)} .
\]

\medskip

\noindent
\textbf{Step 2: Bound on $\nabla f(x_n)^\top b(x_n,y_n)$.}

From~Lemma \ref{lem:drift_g}, the bias term satisfies
\[
    b(x_n,y_n)
    =
    \sum_{w \in \mathcal{A}^c}
    A^{(w)} \sign(-y_n^{(w)})
    -
    \sum_{w \in \mathcal{A}^c}
    A^{(w)}
    \sign\!\big(-{A^{(w)}}^\top \nabla f(x_n)\big).
\]
Therefore,
\begin{equation*}
\begin{split}
    \nabla f(x_n)^\top b(x_n,y_n)
    ={}&
    \sum_{w \in \mathcal{A}^c}
    ({A^{(w)}}^\top \nabla f(x_n))^\top
    \\
    &\quad \times
    \Big(
    \sign(-y_n^{(w)})
    -
    \sign(-{A^{(w)}}^\top \nabla f(x_n))
    \Big).
\end{split}
\end{equation*}

Using the elementary inequality
\[
a\big(\sign(-b)-\sign(-a)\big)
\le
2|b-a|,
\qquad
\forall\, a,b \in \mathbb{R},
\]
coordinatewise, we obtain
\[
\begin{aligned}
    \nabla f(x_n)^\top b(x_n,y_n)
    \le{}&
    2
    \sum_{w \in \mathcal{A}^c}
    \|
    {A^{(w)}}^\top \nabla f(x_n)
    -
    y_n^{(w)}
    \|_1  = \norm{e_n^{(w)}}_1.
\end{aligned}
\]

Combining the bounds obtained in Steps~1 and~2, we arrive at
\[
\begin{aligned}
    \mathbb{E}[f(x_{n+1}) \mid \mathcal{F}_n]
    \le{}&
    f(x_n)
    -
    \alpha_n \eta \|\nabla f(x_n)\|_1
    \\
    &\quad
    +
    2\alpha_n
    \sum_{w \in \mathcal{A}^c}
    \|e_n^{(w)}\|_1
    +
    \frac{L}{2}\alpha_n^2 K_5^2 ,
\end{aligned}
\]
which proves the claim.
\end{proof}

\begin{remark}
Lemma~\ref{descentineq} shows that the objective sequence $\{f(x_n)\}$ satisfies an almost supermartingale-type recursion. Indeed, by Corollary~\ref{Corol:grdioent trachking},
\[
\sum_{n\ge1}
\alpha_n \|e_n^{(w)}\|_1
<
\infty
\qquad \text{a.s.}
\]
for every honest worker $w \in \mathcal{A}^c$. Together with the condition
\[
\sum_{n\ge1} \alpha_n^2 < \infty,
\]
the additional error terms are summable almost surely. Hence, the Robbins--Siegmund theorem applies to the sequence $\{f(x_n)\}$.
\end{remark}

We are now in a position to establish the almost sure boundedness of the global iterates.

\begin{proof}
From Lemma~\ref{descentineq}, we have
\[
\begin{aligned}
    \mathbb{E}[f(x_{n+1}) \mid \mathcal{F}_n]
    \le{}&
    f(x_n)
    -
    \alpha_n \eta \|\nabla f(x_n)\|_1
    \\
    &\quad
    +
    2\alpha_n |\mathcal{A}^c| \|e_n^{(w)}\|_1
    +
    \frac{L}{2}\alpha_n^2 K_5^2 .
\end{aligned}
\]

By Assumption~\ref{Step-sizeA},
\[
\sum_{n\ge1} \alpha_n^2 < \infty.
\]
Moreover, Corollary~\ref{Corol:grdioent trachking} implies that, for each honest worker $w \in \mathcal{A}^c$,
\[
\sum_{n\ge1}
\alpha_n \|e_n^{(w)}\|_1
<
\infty
\qquad \text{a.s.}
\]

Therefore, all positive error terms in the above recursion are almost surely summable. Applying the Robbins--Siegmund theorem (Theorem~\ref{Robbins Siegmund}), we conclude that the sequence $\{f(x_n)\}$ converges almost surely to a finite random variable.

Finally, since $f$ is coercive, every sublevel set of $f$ is bounded. Because the sequence $\{f(x_n)\}$ converges almost surely, it remains almost surely bounded, and consequently the iterates $\{x_n\}$ remain in a bounded sublevel set of $f$. Hence, $\{x_n\}$
is almost surely bounded.

The almost sure boundedness of $\{y_n^{(w)}\}$ now follows directly from Corollary~\ref{Corol:grdioent trachking}.
\end{proof}

\subsection{Proof of Asymptotic Convergence ( Theorem \ref{thm:almost.sure.conv} )}

\textbf{Almost sure convergence: two time-scale analysis.}
In this subsection, we establish the almost sure convergence result using the
two time-scale stochastic recursive inclusion framework of \cite{yaji2020stochastic}.

\begin{proof}[Proof of Theorem~\ref{thm:almost.sure.conv}] The proof of this Theorem follows through the following steps.

\paragraph{Step-1: Rewriting the updates in two--time--scale stochastic approximation form}

To analyze the algorithm using the standard two--time--scale stochastic approximation framework of \cite{yaji2020stochastic}, we first express the recursions \eqref{eq:slow time scaleA} and \eqref{eq:fast time scaleB} in canonical stochastic approximation form. This separates the deterministic drift from the martingale noise and clarifies the limiting dynamics.

\paragraph{Fast time--scale recursion (local iterates).}
For each honest worker \(w\in\mathcal A^c\), the update \eqref{eq:slow time scaleA} can be written as
\begin{equation}
y_{n+1}^{(w)} - y_n^{(w)} - \beta_n M_{n+1}^{(1)}
= \beta_n ( H_1^{(w)}(x_n,y_n) + B_n),
\label{eq:sa_fast_form}
\end{equation}
where \(\{M_{n+1}^{(1)}\}\) is a martingale difference sequence with respect to the filtration \(\{\mathcal F_n\}\).
From \eqref{hwn} and \eqref{C2M}, the drift term is
\begin{equation}
H_1^{(w)}(x_n,y_n)
= {A^{(w)}}^\top \nabla f(x_n) - y_n^{(w)} \quad \text{and} \;  B_n =   {A^{(w)}}^\top \nabla f(x_{n- \tau_n^{(w)}}) - {A^{(w)}}^\top \nabla f(x_n) .
\label{eq:H1_def}
\end{equation}
Note that in view of L-smoothness the bias term $B_n$ satisfies
\begin{equation*}
    \begin{split}
        \norm{B_n} \leq L \Bar{A} \norm{x_n - x_{n-\tau_n^{(w)}}} \leq K_5 \sum_{k=n-\tau_n^{(w)}}^{n-1} \alpha_k. 
    \end{split}
\end{equation*}
Note that in view of \eqref{Delay} $B_n \to 0$ a.s.
The martingale noise satisfies
\begin{equation}
\mathbb{E}\!\left[\|M_{n+1}^{(1)}\|^2 \mid \mathcal F_n\right]
\le (1+K_2)\| y_n - {A^{(w)}}^\top \nabla f(x_{n- \tau_n^{(w)}})\|^2
+ \frac{K_3}{\lambda_n^2}
+ K_4 .
\label{eq:noise_fast}
\end{equation}

\paragraph{Slow time--scale recursion (global iterate).}
The update \eqref{eq:fast time scaleB} can be written as the stochastic recursive inclusion
\begin{equation}
x_{n+1} - x_n - \alpha_n M_{n+1}^{(2)}
\in \alpha_n H_2(x_n,y_n),
\label{eq:sa_slow_form}
\end{equation}
where \(\{M_{n+1}^{(2)}\}\) is a martingale difference sequence.
\\ The mean--field map admits the decomposition (see   Lemma \ref{lem:drift_g})
\begin{equation}
H_2(x_n,y_n)
=
\sum_{w\in\mathcal A} A^{(w)} \xi^{(w)}
+
\sum_{w\in\mathcal A^c}
A^{(w)} \lambda_w(y_n^{(w)}),
\label{eq:H2_def}
\end{equation}
where \(\xi^{(w)}\in[-1,1]^m\) represents arbitrary adversarial actions and
\[
\lambda_w(y)(j)\in
\begin{cases}
\operatorname{sign}\!\big(-{a_j^{(w)}}^\top y\big),
& \text{if } w\in\mathcal A^c \text{ and } {a_j^{(w)}}^\top y\neq 0,
\\[4pt]
[-1,1],
& \text{otherwise}.
\end{cases}
\]

The martingale noise in the slow recursion satisfies
\begin{equation}
\mathbb{E}\!\left[\|M_{n+1}^{(2)}\|^2 \mid \mathcal F_n\right]
\le K,
\qquad K>0.
\label{eq:noise_slow}
\end{equation}

\paragraph{Step-2: Stability of the Iterates} In Theorem \ref{thm:almst boundeness} we prove that $(x_n)$ and $(y_n)$ are almost surely bounded. 
\paragraph{Step-3: Properties of the Noise}
In Lemmas \ref{Corolloary assumption on noise} and \ref{LemmaNoiseSlow}, we establish this property for both the fast and slow time scales.

\paragraph{Fast time--scale noise.}
For the martingale difference sequence \(\{M_{n+1}^{(1)}\}\), we show that for any fixed \(T>0\),
\begin{equation}
\lim_{n\to\infty}
\sup_{\, n\le k\le \tau^1(n,T)}
\left\|
\sum_{m=n}^{k}\beta_m M_{m+1}^{(1)}
\right\|
=0
\qquad \text{a.s.},
\label{eq:noise_fast_vanish}
\end{equation}
where the stopping time \(\tau^1(n,T)\) is defined by
\begin{equation}
\tau^1(n,T)
:=
\min\left\{
m\ge n \;\middle|\;
\sum_{k=n}^{m+1}\beta_k \ge T
\right\}.
\label{eq:tau_fast}
\end{equation}

\paragraph{Slow time--scale noise.}
Similarly, for the martingale difference sequence \(\{M_{n+1}^{(2)}\}\), we prove that for any \(T>0\),
\begin{equation}
\lim_{n\to\infty}
\sup_{\, n\le k\le \tau^2(n,T)}
\left\|
\sum_{m=n}^{k}\alpha_m M_{m+1}^{(2)}
\right\|
=0
\qquad \text{a.s.},
\label{eq:noise_slow_vanish}
\end{equation}
where the stopping time \(\tau^2(n,T)\) is defined as
\begin{equation}
\tau^2(n,T)
:=
\min\left\{
m\ge n \;\middle|\;
\sum_{k=n}^{m+1}\alpha_k \ge T
\right\}.
\label{eq:tau_slow}
\end{equation}
\paragraph{Step-4: Marchaud property of \(H_2(x_n,y_n)\)}

The set-valued map
\[
H_2 : \mathbb{R}^d \times \mathbb{R}^{Nm} \rightrightarrows \mathbb{R}^d,
\]
is a Marchaud map (see Lemma~\ref{lemmarchaud}).
\paragraph{Step-5: Fast time--scale analysis}

We have already established in Theorem~\ref{Corol:grdioent trachking} that the fast iterates \(y_n^{(w)}\) track almost surely \({A^{(w)}}^\top \nabla f(x_n)\).

\paragraph{Step-6: Slow time--scale analysis and limiting dynamics}

We now turn to the evolution of the slow variable \(\{x_n\}\). The tracking result from Step~1 implies that, asymptotically,
\[
y_n^{(w)} \approx {A^{(w)}}^\top \nabla f(x_n),
\qquad \text{for all } w \in \mathcal{A}^c.
\]
Substituting this relation into the slow-time-scale recursion and observing that
\[
b\!\left(x_n,{A^{(w)}}^\top \nabla f(x_n)\right)=0,
\]
the update for \(x_n\) can be expressed as
\[
x_{n+1}
=
x_n
+
\alpha_n\big(g'(x_n)+M_{n+1}^{(2)}\big),
\]
where the (possibly set-valued) drift \(g'(\cdot)\) is given by
\[
g'(x)
:=
\sum_{w \in \mathcal{A}}
A^{(w)}\xi^{(w)}
+
\sum_{w \in \mathcal{A}^c}
A^{(w)}
\operatorname{sign}\!\big(-{A^{(w)}}^\top \nabla f(x)\big),
\]
with \(\xi^{(w)} \in [-1,1]^m\) representing arbitrary adversarial actions.

By construction, \(g'(x)\in H(x)\), where the set-valued map \(H\) is defined as
\[
H(x)
:=
\sum_{w=1}^{N} A^{(w)} \lambda_w,
\]
with
\[
\lambda_w(j)\in
\begin{cases}
\operatorname{sign}\!\big(-{a_j^{(w)}}^\top \nabla f(x)\big),
& \text{if } w\in\mathcal{A}^c \text{ and } {a_j^{(w)}}^\top \nabla f(x)\neq 0,
\\[0.5ex]
[-1,1],
& \text{otherwise}.
\end{cases}
\]

Consequently, the asymptotic behavior of the sequence \(\{x_n\}\) is governed by the ordinary differential inclusion
\begin{equation}
\dot{x}(t)\in H(x(t)).
\label{eqodi}
\end{equation}

Lemma~\ref{Marchaud} shows that \(H\) is a Marchaud map, ensuring well-posedness of \eqref{eqodi}, while Lemma~\ref{ref:odistability} establishes that every Carath\'eodory solution of \eqref{eqodi} converges asymptotically to the  set of stationary points \(\mathcal{X}^\ast\).

Steps~1--6 together verify that the two--time--scale stochastic recursive inclusion satisfies the required conditions for the convergence theory of stochastic approximation \cite{yaji2020stochastic}. Therefore, invoking the stochastic approximation framework of~\cite{yaji2020stochastic} (in particular, Theorem~5.9 therein), we conclude that the iterates \(\{x_n\}\) converge almost surely to \(\mathcal{X}^\ast\).

\end{proof}

To complete the proof of Theorem~\ref{thm:almost.sure.conv}, we need to establsh the following key technical Lemmas.
\begin{lemma}
For any $T>0$,
\[
\lim_{n \to \infty}
\sup_{\, n \le k \le \tau^1(n,T)}
\left\|
\sum_{m=n}^{k} \beta_m M_{m+1}^{(1)}
\right\|
= 0 \quad \text{a.s.},
\]
where
\[
\tau^1(n,T)
:= \min \left\{ m \ge n \;\middle|\; \sum_{k=n}^{m+1} \beta_k \ge T \right\}.
\]
\label{Corolloary assumption on noise}
\end{lemma}

\begin{proof}
Define
\[
S_n := \sum_{k=1}^{n} \beta_k M_{k+1}^{(1)}, \qquad n \ge 1 .
\]
Then $\{S_n\}$ is a square-integrable martingale with respect to
$\{\mathcal{F}_n\}$.
\\ Since $S_{n+1}-S_n=\beta_n M_{n+1}^{(1)}$, we have
\begin{equation}
\begin{aligned}
\sum_{n \ge 1}
\mathbb{E}\!\left[
\| S_{n+1} - S_n \|^2 \mid \mathcal{F}_n
\right]
&= \sum_{n \ge 1} \beta_n^2
\mathbb{E}\!\left[
\| M_{n+1}^{(1)} \|^2 \mid \mathcal{F}_n
\right] \\
&\le \sum_{n \ge 1} \beta_n^2
\Bigg(
(1+K_2)
\big\|
{A^{(w)}}^\top \nabla f(x_{n- \tau_n^{(w)}}) - y_n^{(w)}
\big\|^2
+ \frac{K_3}{\lambda_n^2}
+ K_4
\Bigg).
\end{aligned}
\end{equation}
By the almost sure boundedness of
$e_n^{(w)} := {A^{(w)}}^\top \nabla f(x_{n}) - y_n^{(w)}$
and Assumption~\ref{Step-sizeA},
\[
\sum_{n \ge 1} \beta_n^2 < \infty,
\qquad
\sum_{n \ge 1} \frac{\beta_n^2}{\lambda_n^2} < \infty.
\]
Hence,
\[
\sum_{n \ge 1}
\mathbb{E}\!\left[
\| S_{n+1} - S_n \|^2 \mid \mathcal{F}_n
\right]
< \infty
\quad \text{a.s.}
\]
and the martingale convergence theorem implies that $S_n$ converges almost surely.
\\ Define the tail sequence
\[
T_n := \sum_{k=n}^{\infty} \beta_k M_{k+1}^{(1)}, \qquad n \ge 1 .
\]
Then $T_n \to 0$ almost surely. For any $k \ge n$,
\[
\sum_{m=n}^{k} \beta_m M_{m+1}^{(1)}
= T_n - \sum_{m=k+1}^{\infty} \beta_m M_{m+1}^{(1)} .
\]
Taking norms,
\[
\sup_{n \le k \le \tau^1(n,T)}
\left\|
\sum_{m=n}^{k} \beta_m M_{m+1}^{(1)}
\right\|
\le \|T_n\| + \sup_{m \ge n} \|T_m\|.
\]
Letting $n \to \infty$ yields the claim.
\end{proof}

\begin{lemma}
For any $T>0$,
\[
\lim_{n \to \infty}
\sup_{\, n \le k \le \tau^2(n,T)}
\left\|
\sum_{m=n}^{k} \alpha_m M_{m+1}^{(2)}
\right\|
= 0 \quad \text{a.s.},
\]
where
\[
\tau^2(n,T)
:= \min \left\{ m \ge n \;\middle|\; \sum_{k=n}^{m+1} \alpha_k \ge T \right\}.
\]
\label{LemmaNoiseSlow}
\end{lemma}

\begin{proof}
The proof follows along the same lines as Lemma~\ref{Corolloary assumption on noise},
with $\alpha_n$ in place of $\beta_n$ and $M_{n+1}^{(2)}$ in place of
$M_{n+1}^{(1)}$.
\end{proof}

\begin{lemma}
    The map $H_2: \mathbb{R}^d \times \mathbb{R}^{m N} \rightrightarrows \mathbb{R}^d$ is Marchaud. 
    \label{lemmarchaud}
\end{lemma}

\begin{proof}
    The proof is similar to Lemma \ref{Marchaud} and we skip the details.
\end{proof}

\begin{lemma}
    The set-valued map $H: \mathbb{R}^d \rightrightarrows \mathbb{R}^d$ is a Marchaud map.
    \label{Marchaud}
\end{lemma}
\begin{proof}
We prove the three required properties of \(H\).

\medskip
\noindent\textbf{Step 1: Convexity and compactness.}  
For fixed \(x\), each coordinate set defining \(H(x)\) is convex and compact (either a singleton \(\{\pm1\}\) or the interval \([-1,1]\)). Hence \(H(x)\) is a finite Minkowski sum of convex compact sets, and is therefore convex and compact.

\medskip
\noindent\textbf{Step 2: Boundedness.}  
Every coordinate of any \(h\in H(x)\) lies in \([-N,N]\); thus \(\|h\|_\infty\le N\), and in particular \(\|h\|\) is uniformly bounded.

\medskip
\noindent\textbf{Step 3: Closedness of the graph (upper semicontinuity).}  
Let \(x_n\to x\) and \(h_n\in H(x_n)\) with \(h_n\to h\). By construction there exist vectors \(\lambda_{w,n}\in[-1,1]^m\) such that
\[
h_n=\sum_{w=1}^N A^{(w)}\lambda_{w,n},\qquad n\ge0.
\]
By compactness of \([-1,1]^m\) we may (passing to a subsequence if necessary) assume \(\lambda_{w,n}\to\xi_w\in[-1,1]^m\) for each \(w\). We claim that the limit \(\xi_w\) satisfies the coordinate--wise rules defining \(H(x)\), namely
\[
\xi_w(j)\in
\begin{cases}
\operatorname{sign}\!\big(-{a_j^{(w)}}^\top \nabla f(x)\big), &
w\in\mathcal A^c,\ {a_j^{(w)}}^\top \nabla f(x)\neq0,\\[4pt]
[-1,1], & \text{otherwise}.
\end{cases}
\]
Indeed, if \(w\in\mathcal A^c\) and \({a_j^{(w)}}^\top\nabla f(x)\neq0\), then by continuity of \(\nabla f\) there exists \(N_0\) such that for all \(n\ge N_0\) the sign \(\operatorname{sign}(-{a_j^{(w)}}^\top\nabla f(x_n))\) equals \(\operatorname{sign}(-{a_j^{(w)}}^\top\nabla f(x))\); hence \(\lambda_{w,n}(j)\) is eventually constant and the limit \(\xi_w(j)\) equals that sign. If \({a_j^{(w)}}^\top\nabla f(x)=0\) (or \(w\in\mathcal A\)), then each \(\lambda_{w,n}(j)\in[-1,1]\) and therefore \(\xi_w(j)\in[-1,1]\).  

Thus \(\xi_w\) satisfies the defining coordinate conditions at \(x\), and
\[
h=\lim_{n\to\infty}h_n=\sum_{w=1}^N A^{(w)}\xi_w\in H(x).
\]
Therefore the graph of \(H\) is closed, i.e. \(H\) is upper semicontinuous. This completes the proof.
\end{proof}

\begin{remark}
Lemma~\ref{Marchaud} is crucial for establishing the asymptotic convergence of
\(\{x_n\}\) via \eqref{eqodi}, as it guarantees well--posedness of the limiting
differential inclusion. In particular, for any initial condition
\(x_0\in\mathbb{R}^d\), there exists a Carath\'eodory solution \(x(\cdot)\) that is
absolutely continuous and satisfies \(\dot{x}(t)\in H(x(t))\) almost everywhere.
This existence result enables the asymptotic stability analysis in
Lemma~\ref{ref:odistability}, which ultimately yields convergence $x_n$ to the set of stationary points.
\end{remark}

\begin{lemma}
\label{ref:odistability}
Consider the differential inclusion
\[
    \dot{x}(t) \in H(x(t)),
\]
where $H$ is the set-valued map defined in~\eqref{eqodi}. Then every
Carath\'eodory solution $x(t)$ of this differential inclusion converges
asymptotically to the set of stationary points
\[
    \mathcal{X}^\ast
    :=
    \{ x \in \mathbb{R}^d \mid \nabla f(x) = 0 \}.
\]
\end{lemma}

\begin{proof}
We analyze the asymptotic behavior of the differential inclusion using a
Lyapunov argument. Define
\[
    V(x) := f(x) - f^\ast .
\]
The set-valued Lie derivative of $V$ with respect to the inclusion
$\dot{x} \in H(x)$ is
\[
    \widetilde{\mathcal{L}} V(x)
    :=
    \big\{
    a \in \mathbb{R}
    \;\big|\;
    \exists\, v \in H(x)
    \text{ such that }
    a = \nabla V(x)^\top v
    \big\}.
\]

Fix any $a \in \widetilde{\mathcal{L}} V(x)$. Then there exists $v \in H(x)$ such
that
\[
    a = \nabla f(x)^\top v .
\]
By definition of $H$, the vector $v$ admits the representation
\[
    v
    =
    \sum_{w=1}^{N} \sum_{j=1}^m a_j^{(w)} \lambda_w(j),
\]
which yields
\[
\begin{aligned}
    a
    &=
    \sum_{w \in \mathcal{A}^c}
    \sum_{j=1}^m
    \nabla f(x)^\top a_j^{(w)} \lambda_w(j)
    +
    \sum_{w \in \mathcal{A}}
    \sum_{j=1}^m
    \nabla f(x)^\top a_j^{(w)} \lambda_w(j).
\end{aligned}
\]

For honest workers $w \in \mathcal{A}^c$, we have
\[
    \lambda_w(j)
    =
    \operatorname{sign}\!\big(-{a_j^{(w)}}^\top \nabla f(x)\big),
\]
while for adversarial workers $w \in \mathcal{A}$, $\lambda_w(j) \in [-1,1]$. It
therefore follows that
\begin{equation}
\begin{split}
    a
    &\le
    - \sum_{w \in \mathcal{A}^c}
    \|{A^{(w)}}^\top \nabla f(x)\|_1
    +
    \sum_{w \in \mathcal{A}}
    \|{A^{(w)}}^\top \nabla f(x)\|_1
    \\
    &\le
    -\eta \|\nabla f(x)\|_1
    \le 0,
\end{split}
\label{eq:lie-derivative-bound}
\end{equation}
where $\eta>0$ is the constant defined earlier. Consequently,
\[
    \sup_{a \in \widetilde{\mathcal{L}} V(x)} a \le 0.
\]

By Proposition~10 of~\cite{cortes2008discontinuous1}, this implies that
\begin{equation}
    \frac{d}{dt} f(x(t)) \le 0
    \quad \text{for almost all } t \ge 0,
    \label{eq:monotone-f}
\end{equation}
that is, the function $f(x(t))$ is non-increasing along any Carath\'eodory
solution.

Let $f(x(0)) = c$ and define the sublevel set
\[
    S := \{ x \in \mathbb{R}^d \mid f(x) \le c \}.
\]
By~\eqref{eq:monotone-f}, the set $S$ is positively invariant. Moreover, $S$ is
compact by Assumption~\ref{L-Smooth}. We may therefore invoke the invariance
principle for differential inclusions (Theorem~4 of~\cite{cortes2008discontinuous1})
to conclude that every Carath\'eodory solution converges asymptotically to the
largest invariant set contained in
\[
    S \cap \overline{\{ x \in \mathbb{R}^d \mid 0 \in \widetilde{\mathcal{L}} V(x) \}}.
\]

Finally, from~\eqref{eq:lie-derivative-bound} we see that
$0 \in \widetilde{\mathcal{L}} V(x)$ if and only if $\nabla f(x) = 0$, which implies
$x \in \mathcal{X}^\ast$. This completes the proof.
\end{proof}

\begin{remark}
Combining Proposition~\ref{prop:enbound} and Corollary~\ref{Corol:grdioent trachking} with the Robbins--Siegmund argument, we have shown that the iterates are almost surely bounded and that the fast variables \(y_n^{(w)}\) track the moving equilibria \({A^{(w)}}^\top\nabla f(x_n)\).  Using this tracking together with Lemmas~\ref{Marchaud} and \ref{ref:odistability} (ODE/ODI stability) we concluded in Theorem~\ref{thm:almost.sure.conv} that every limit point of \(\{x_n\}\) lies in the stationary set \(\mathcal{X}^\ast\).  

The two key mechanisms behind this result are (i) time--scale separation, which allows the fast dynamics to equilibrate around the directional gradients, and (ii) the ODI formulation, which handles non-smooth and adversarial contributions via a Marchaud map and an invariance principle.  

In the next section we quantify these statements: we derive explicit  convergence rates that depend explicitly on the step-size schedules and problem parameters.
\end{remark}

\section{Proof of Convergence Rate (Theorem \ref{thm:rate_optimalfirstorder} and Theorem \ref{thm:rate_optimal})}

In this section, we establish the main convergence theorem for the two--time--scale stochastic recursion defined by \eqref{eq:slow time scaleA} and \eqref{eq:fast time scaleB}, where the stochastic drift terms satisfy \eqref{hwn}, \eqref{C2M}, \eqref{Tgn}, and \eqref{rysh}. Subsequently, using the constants summarized in Lemma \ref{lem:approximated gradientf}, Lemma \ref{lem:approximated gradient}, and Lemma \ref{lem:drift_g} we derive the convergence rates for the synchronous and asynchronous settings, covering both first-order and zeroth-order cases.

\subsection{Tracking Rate}
\begin{proposition}
\label{propoenbound 1}
Suppose Assumptions~\ref{L-Smooth}, \ref{noise-first order},
\ref{ass:delay}, and~\ref{fawzi-condn} hold. Assume further that the step-size sequence $\{\beta_n\}$ satisfies
\[
\beta_n
<
\frac{1}{2(4K_2+1)}
\qquad
\forall\, n \in \mathbb{N}.
\]
Then the stochastic error sequence $\{e_n^{(w)}\}$ for each honest worker $w \in \mathcal{A}^c$ satisfies the following conditional mean-square bound:
\begin{equation}
\begin{split}
\mathbb{E}\!\left[
\|e_{n+1}^{(w)}\|^2
\mid
\mathcal{F}_n
\right]
\le{}&
\big(1-\beta_n\big)
\|e_n^{(w)}\|^2
+
R_n
+
D_n ,
\end{split}
\label{18eq1}
\end{equation}
where
\[
R_n
:=
C_n
+
\frac{B_n^2}{\beta_n},
\]
with
\[
B_n
:=
2\bar{A} L K_5 \alpha_n,
\]
and the quantities $C_n$ and $D_n$ are defined in Proposition~\ref{prop:enbound}.
\end{proposition}
\begin{proof}
From Proposition~\ref{prop:enbound}, for each honest worker
$w \in \mathcal{A}^c$, we have
\begin{equation}
\begin{split}
\mathbb{E}\!\left[
\|e_{n+1}^{(w)}\|^2
\mid
\mathcal{F}_n
\right]
\le{}&
\big(1-A_n\big)\|e_n^{(w)}\|^2
+
B_n \|e_n^{(w)}\|
+
C_n
+
D_n ,
\end{split}
\label{proofc1}
\end{equation}
where
\[
A_n
:=
2\beta_n
-
c_2\beta_n
-
4K_2\beta_n^2
-
\beta_n^2,
\qquad c_2>0,
\]
\[
B_n
:=
2\bar{A} L K_5 \alpha_n,
\]
\[
C_n
:=
\frac{2K_3\beta_n^2}{\lambda_n^2}
+
\frac{\beta_n}{c_2}K_1^2\lambda_n^2
+
2\alpha_n^2 L^2 \bar{A}^2 K_5^2
+
2K_4\beta_n^2,
\]
and
\[
\begin{aligned}
D_n
:={}
4K_2L^2\bar{A}^2K_5^2
\beta_n^2
\tau_n^{(w)}
\sum_{k=n-\tau_n^{(w)}}^{n-1}
\alpha_k^2
+
L^2\bar{A}^2K_5^2
\tau_n^{(w)}
\sum_{k=n-\tau_n^{(w)}}^{n-1}
\alpha_k^2 .
\end{aligned}
\]

The only term on the right-hand side of~\eqref{proofc1} that is linear in
$\|e_n^{(w)}\|$ is
\[
B_n \|e_n^{(w)}\|.
\]
We control this term using Young's inequality. Specifically, for any $c_2>0$,
\[
ab
\le
c_2\beta_n a^2
+
\frac{b^2}{4c_2\beta_n},
\qquad
a,b \ge 0.
\]
Applying this inequality with
\[
a=\|e_n^{(w)}\|,
\qquad
b=B_n,
\]
yields
\[
B_n \|e_n^{(w)}\|
\le
c_2\beta_n \|e_n^{(w)}\|^2
+
\frac{B_n^2}{4c_2\beta_n}.
\]

Substituting the above estimate into~\eqref{proofc1}, we obtain
\begin{equation*}
\begin{split}
\mathbb{E}\!\left[
\|e_{n+1}^{(w)}\|^2
\mid
\mathcal{F}_n
\right]
\le{}&
\Big(
1
-
2\beta_n
+
c_2\beta_n
+
4K_2\beta_n^2
+
\beta_n^2
\Big)
\|e_n^{(w)}\|^2
\\
&\quad
+
\frac{B_n^2}{4c_2\beta_n}
+
C_n
+
D_n .
\end{split}
\end{equation*}

Now choose
\[
c_2=\frac14.
\]
Then
\[
1
-
2\beta_n
+
c_2\beta_n
+
4K_2\beta_n^2
+
\beta_n^2
=
1
-
\beta_n
+
\beta_n
\Big(
-\frac34
+
(4K_2+1)\beta_n
\Big).
\]

Since
\[
\beta_n
<
\frac{1}{2(4K_2+1)},
\]
it follows that
\[
-\frac34
+
(4K_2+1)\beta_n
<
-\frac14.
\]
Consequently,
\[
1
-
2\beta_n
+
c_2\beta_n
+
4K_2\beta_n^2
+
\beta_n^2
\le
1-\beta_n.
\]

Finally, observing that
\[
\frac{B_n^2}{4c_2\beta_n}
=
\frac{B_n^2}{\beta_n},
\]
we conclude that
\[
\mathbb{E}\!\left[
\|e_{n+1}^{(w)}\|^2
\mid
\mathcal{F}_n
\right]
\le
(1-\beta_n)\|e_n^{(w)}\|^2
+
R_n
+
D_n,
\]
where
\[
R_n
=
C_n
+
\frac{B_n^2}{\beta_n}.
\]
This completes the proof.
\end{proof}

\begin{proposition}[Finite-horizon bound for the tracking error]
\label{thm:finite_horizon_alpha_beta_lambda}
Under the assumptions of Theorem~\ref{thm:rate_optimal}, let
\[
\beta_n=\frac{\beta}{(1+n)^b},\qquad
\alpha_n=\frac{\alpha}{(1+n)^a},\qquad
\lambda_n=\frac{\lambda}{(1+n)^p},
\]
where $\alpha,\beta,\lambda>0$ and $a,b,p>0$. Then, for every honest worker
$w\in\mathcal{A}^c$  the following holds.

In particular,
\[
\mathbb{E}\!\left[\|e_{n+1}^{(w)}\|^2\right]
\le
\Bigg(\prod_{k=0}^{n}(1-\beta_k)\Bigg)\mathcal{E}
+
\sum_{j=1}^5
\frac{3 C_j 2^{q_j}(n+1)^b}{\beta (n+2)^{q_j}}
+
o\!\left(\frac{1}{(n+1)^r}\right),
\qquad \forall\, r>0.
\]
\end{proposition}
where \(\mathbb{E}\!\left[\|e_0^{(w)}\|^2\right]\le \mathcal E\).
The constants $C_j, q_j$ are given in the following table \ref{tab:constants_tracking_error}.

\begin{table}[t]
\centering
\caption{Constants appearing in the finite-horizon tracking-error bound.}
\label{tab:constants_tracking_error}
\renewcommand{\arraystretch}{1.25}
\begin{tabular}{c c c}
\hline
\textbf{Index} & \textbf{Constant} & \textbf{Exponent} \\
\hline
$1$ &
$\displaystyle C_1 = \frac{2K_3\beta^2}{\lambda^2}$ &
$\displaystyle q_1 = 2b-2p$ \\[0.8em]
$2$ &
$\displaystyle C_2 = \frac{\beta K_1^2\lambda^2}{c_2}$ &
$\displaystyle q_2 = b+2p$ \\[0.8em]
$3$ &
$\displaystyle C_3 = \Big(2L^2\bar A^2K_5^2+C_D\,\mathbb{E}[\tau^2]\Big)\alpha^2$ &
$\displaystyle q_3 = 2a$ \\[0.8em]
$4$ &
$\displaystyle C_4 = 2K_4\beta^2$ &
$\displaystyle q_4 = 2b$ \\[0.8em]
$5$ &
$\displaystyle C_5 = \frac{4\bar A^2L^2K_5^2\alpha^2}{\beta}$ &
$\displaystyle q_5 = 2a-b$ \\
\hline
\end{tabular}
\end{table}

\begin{proof}
From Proposition~\ref{propoenbound 1}, taking total expectation yields
\begin{equation}
\label{eq:total_exp_recap}
\mathbb{E}\!\left[\|e_{n+1}^{(w)}\|^2\right]
\le
(1-\beta_n)\mathbb{E}\!\left[\|e_n^{(w)}\|^2\right]
+ R_n + \mathbb{E}[D_n],
\end{equation}
where
\[
B_n = 2\bar A L K_5 \alpha_n,
\qquad
R_n = C_n + \frac{B_n^2}{\beta_n},
\]
and
\[
C_n
=
\frac{2K_3\beta_n^2}{\lambda_n^2}
+
\frac{\beta_n}{c_2}K_1^2\lambda_n^2
+
2\alpha_n^2 L^2 \bar{A}^2 K_5^2
+
2K_4\beta_n^2.
\]
Hence
\[
R_n
=
\frac{2K_3\beta_n^2}{\lambda_n^2}
+
\frac{\beta_n}{c_2}K_1^2\lambda_n^2
+
2\alpha_n^2 L^2 \bar{A}^2 K_5^2
+
2K_4\beta_n^2
+
\frac{4\bar A^2 L^2 K_5^2\,\alpha_n^2}{\beta_n}.
\]

Moreover, we have
\begin{equation}
\mathbb{E}[D_n]
=
O(\alpha_n^2)\,\mathbb{E}[\tau^2].
\label{Delaynew assumption}
\end{equation}
Therefore, after absorbing the delay contribution into the \((1+n)^{-2a}\) term, we may write
\[
\mathbb{E}[D_n]
\le
C_D \,\mathbb{E}[\tau^2]\,\alpha_n^2
\]
for some constant \(C_D>0\). Consequently,
\begin{equation}
\label{eq:total_driver}
\mathbb{E}\!\left[\|e_{n+1}^{(w)}\|^2\right]
\le
(1-\beta_n)\mathbb{E}\!\left[\|e_n^{(w)}\|^2\right]
+
\widetilde R_n,
\end{equation}
where
\[
\widetilde R_n
:=
R_n + C_D\,\mathbb{E}[\tau^2]\alpha_n^2.
\]

Now choose the polynomial step-size schedules
\[
\beta_n=\frac{\beta}{(1+n)^b},
\qquad
\alpha_n=\frac{\alpha}{(1+n)^a},
\qquad
\lambda_n=\frac{\lambda}{(1+n)^p},
\]
with \(\alpha,\beta,\lambda>0\) and \(a,b,p>0\). Then
\begin{equation}
\label{Rn-final}
\begin{aligned}
\widetilde R_n
\le{}&
\frac{2K_3\beta^2}{\lambda^2}(1+n)^{-2b+2p}
+
\frac{\beta K_1^2\lambda^2}{c_2}(1+n)^{-b-2p}
\\
&\quad
+
\Big(2L^2\bar A^2K_5^2
+
C_D\,\mathbb{E}[\tau^2]\Big)\alpha^2(1+n)^{-2a}
+
2K_4\beta^2(1+n)^{-2b}
\\
&\quad
+
\frac{4\bar A^2 L^2 K_5^2\alpha^2}{\beta}(1+n)^{-2a+b}.
\end{aligned}
\end{equation}

Iterating~\eqref{eq:total_driver} from \(0\) to \(n\), we obtain
\begin{equation}
\label{eq:iterated_exact}
\mathbb{E}\!\left[\|e_{n+1}^{(w)}\|^2\right]
\le
\Bigg(\prod_{k=0}^{n}(1-\beta_k)\Bigg)\mathbb{E}\!\left[\|e_0^{(w)}\|^2\right]
+
\sum_{i=0}^{n}
\Bigg(\prod_{k=i+1}^{n}(1-\beta_k)\Bigg)\widetilde R_i.
\end{equation}

Substituting~\eqref{Rn-final} into~\eqref{eq:iterated_exact} gives the final polynomial-decay form
\begin{equation}
\begin{split}
& \mathbb{E}\!\left[\|e_{n+1}^{(w)}\|^2\right]
 \le
\Bigg(\prod_{k=0}^{n}(1-\beta_k)\Bigg)\mathbb{E}\!\left[\|e_0^{(w)}\|^2\right]
 +
\sum_{i=0}^{n}
\Bigg(\prod_{k=i+1}^{n}(1-\beta_k)\Bigg)
\Big(
C_1(1+i)^{-2b+2p}
+ \\ &
C_2(1+i)^{-b-2p}
 +
C_3(1+i)^{-2a}
+
C_4(1+i)^{-2b}
+
C_5(1+i)^{-2a+b}
\Big),
\end{split}
\end{equation}
for suitable positive constants \(C_1,C_2,C_3,C_4,\) and \(C_5\) defined in Table \ref{tab:constants_tracking_error}. Then by applying Proposition \ref{Prop:Bigo} we obtain the result.
\end{proof}
The following Corollary is now immediate.

\begin{corollary}[Mean-square tracking rate]
Under the step-size choice of Proposition~\ref{thm:finite_horizon_alpha_beta_lambda},
assume that
$q_i>b \qquad \text{for all } i\in\mathcal I,$
where
\[
\mathcal I := \{\, i\in\{1,2,3,4,5\} \mid C_i>0 \,\}.
\]
Define
\[
\gamma := \min_{i\in\mathcal I}(q_i-b) > 0.
\]
Then there exists a constant \(\widetilde{C_1}>0\), depending on \(\mathcal E\) and the nonzero constants
\(C_1,\dots,C_5\), such that
\[
\mathbb{E}\!\left[\|e_{n+1}^{(w)}\|^2\right]
\le
\widetilde{C}_1\,(1+n)^{-\gamma},
\qquad \forall\, n\ge0.
\]
\label{mean square tracking raye}
\end{corollary}

\begin{remark}
    Note that from Table \ref{tab:constants_tracking_error} if $\alpha = \frac{1}{mN}$ and $\beta = \frac{1}{m N}$ then it can be verified that $\widetilde{C_1} = \mathcal{O} (m^2 N^2)$.
\end{remark}

\subsection{Convergence Rate Proof}

\begin{proposition}
\label{prop:avg_grad_upper_bound}
For every $n \ge 1$, define
\[
S_n := \sum_{k=1}^n \alpha_k.
\]
Then, there exists a constant $\widetilde C_1 > 0$, independent of $n$, $m$, and $N$, such that
\[
\begin{aligned}
\frac{1}{S_n}\sum_{k=1}^n \alpha_k\,\mathbb{E}\|\nabla f(x_k)\|_1
\le{}&
\frac{\mathbb{E}[f(x_1)]-\inf f}{\eta S_n}
+ \frac{2|\mathcal{A}^c|}{\eta}\,
\frac{\sqrt{\widetilde C_1}\,m^{1/2}\sum\limits_{k=1}^n \alpha_k(1+k)^{-\gamma/2}}{S_n}
\\
&\quad
+ \frac{L K_5^2}{2\eta}\,
\frac{\sum\limits_{k=1}^n \alpha_k^2}{S_n},
\end{aligned}
\]
where $\gamma$ is defined in Corollary~\ref{mean square tracking raye}.
\end{proposition}

\begin{proof}
The result follows from the descent inequality established in Lemma~\ref{descentineq}. In particular, for every $n \ge 1$, we have
\[
\begin{aligned}
\mathbb{E}[f(x_{n+1}) \mid \mathcal{F}_n]
\le{}&
f(x_n)
- \alpha_n \eta \|\nabla f(x_n)\|_1
+ 2\alpha_n |\mathcal{A}^c| \|e_n^{(w)}\|_1
+ \frac{L}{2}\alpha_n^2 K_5^2 .
\end{aligned}
\]

Taking total expectation on both sides yields
\[
\begin{aligned}
\mathbb{E}[f(x_{n+1})]
\le{}&
\mathbb{E}[f(x_n)]
- \alpha_n \eta \,\mathbb{E}\|\nabla f(x_n)\|_1
+ 2\alpha_n |\mathcal{A}^c| \,\mathbb{E}\|e_n^{(w)}\|_1
+ \frac{L}{2}\alpha_n^2 K_5^2 .
\end{aligned}
\]

Rearranging the above inequality, we obtain
\[
\begin{aligned}
\alpha_n \eta\,\mathbb{E}\|\nabla f(x_n)\|_1
\le{}&
\mathbb{E}[f(x_n)]
- \mathbb{E}[f(x_{n+1})]
+ 2\alpha_n |\mathcal{A}^c| \,\mathbb{E}\|e_n^{(w)}\|_1
+ \frac{L}{2}\alpha_n^2 K_5^2 .
\end{aligned}
\]

Summing from $k=1$ to $n$, we get
\[
\begin{aligned}
\sum_{k=1}^n \alpha_k \eta\,\mathbb{E}\|\nabla f(x_k)\|_1
\le{}&
\sum_{k=1}^n
\left(
\mathbb{E}[f(x_k)]
-
\mathbb{E}[f(x_{k+1})]
\right)
+ 2|\mathcal{A}^c|
\sum_{k=1}^n
\alpha_k \mathbb{E}\|e_k^{(w)}\|_1
+ \frac{L}{2}K_5^2
\sum_{k=1}^n \alpha_k^2 .
\end{aligned}
\]

The first term on the right-hand side telescopes, and hence
\[
\begin{aligned}
\sum_{k=1}^n \alpha_k \eta\,\mathbb{E}\|\nabla f(x_k)\|_1
\le{}&
\mathbb{E}[f(x_1)]
-
\mathbb{E}[f(x_{n+1})]
+ 2|\mathcal{A}^c|
\sum_{k=1}^n
\alpha_k \mathbb{E}\|e_k^{(w)}\|_1
+ \frac{L}{2}K_5^2
\sum_{k=1}^n \alpha_k^2 .
\end{aligned}
\]
Dividing both sides by $S_n := \sum_{k=1}^n \alpha_k,$
we obtain
\begin{equation}
\label{eq:avg_grad_bound_clean}
\begin{aligned}
\frac{1}{S_n}
\sum_{k=1}^n
\alpha_k \mathbb{E}\|\nabla f(x_k)\|_1
\le{}&
\frac{\mathbb{E}[f(x_1)] - \mathbb{E}[f(x_{n+1})]}{\eta S_n}
+ \frac{2|\mathcal{A}^c|}{\eta}
\frac{\sum_{k=1}^n \alpha_k \mathbb{E}\|e_k^{(w)}\|_1}{S_n}
\\
&\quad
+ \frac{L K_5^2}{2\eta}
\frac{\sum_{k=1}^n \alpha_k^2}{S_n}.
\end{aligned}
\end{equation}

We next estimate the tracking-error term. By Jensen's inequality,
\[
\mathbb{E}\|e_k^{(w)}\|
\le
\sqrt{\mathbb{E}\|e_k^{(w)}\|^2}.
\]
Using Corollary~\ref{mean square tracking raye}, we further obtain
\[
\mathbb{E}\|e_k^{(w)}\|
\le
\sqrt{\widetilde C_1}\,(1+k)^{-\gamma/2}.
\]

Since $e_k^{(w)} \in \mathbb{R}^m$, the norm equivalence relation
\[
\|v\|_1 \le \sqrt{m}\,\|v\|_2
\]
implies that
\[
\begin{aligned}
\mathbb{E}\|e_k^{(w)}\|_1
\le{}&
\sqrt{m}\,\mathbb{E}\|e_k^{(w)}\|
\le{}
\sqrt{\widetilde C_1}\,m^{1/2}(1+k)^{-\gamma/2}.
\end{aligned}
\]

Therefore,
\[
\sum_{k=1}^n
\alpha_k \mathbb{E}\|e_k^{(w)}\|_1
\le
\sqrt{\widetilde C_1}\,m^{1/2}
\sum_{k=1}^n
\alpha_k(1+k)^{-\gamma/2}.
\]

Substituting the above estimate into \eqref{eq:avg_grad_bound_clean}, and using the fact that
\[
\mathbb{E}[f(x_{n+1})] \ge \inf f,
\]
completes the proof.
\end{proof}

The next Corollary is now immediate. 

\begin{corollary}[Explicit constants in the ergodic gradient bound]
\label{cor:ergodic_constants}
Suppose the assumptions of Theorem~\ref{thm:rate_optimal} hold. Then,
\[
\frac{1}{S_n}\sum_{k=1}^n \alpha_k\,\mathbb{E}\bigl\|\nabla f(x_k)\bigr\|_1
\le
\frac{\mathrm{I}_1}{\eta S_n}
+
D_1\, m^{1/2}\,(1+n)^{-s_1}
+
D_2\,(1+n)^{-s_2},
\]
where
\[
\mathrm{I}_1 := \mathbb{E}[f(x_1)]-\inf f,
\]
and
\[
s_1
:=
\min\!\Bigl\{
\frac{\gamma}{2},
\,1-a
\Bigr\},
\qquad
s_2
:=
\min\{a,\,1-a\}.
\]
The constants $D_1$ and $D_2$ are given by
\[
D_1
=
\frac{2|\mathcal{A}^c|}{\eta}\sqrt{\widetilde C_1},
\qquad
D_2
=
\frac{L K_5^2 \alpha}{2\eta}.
\]
\end{corollary}

\begin{proof}
Starting from the estimate obtained in
Proposition~\ref{prop:avg_grad_upper_bound}, we have
\begin{equation}
\begin{aligned}
\frac{1}{S_n}\sum_{k=1}^n \alpha_k\,\mathbb{E}\|\nabla f(x_k)\|_1
\le{}&
\frac{\mathbb{E}[f(x_1)]-\inf f}{\eta S_n}
+ \underbrace{\frac{2|\mathcal{A}^c|}{\eta}\,
\frac{\sqrt{\widetilde C_1}\,m^{1/2}\,\sum_{k=1}^n \alpha_k(1+k)^{-\gamma/2}}{S_n}}_{\mathbf{M}_1}
\\
&\qquad
+ \underbrace{\frac{L K_5^2}{2\eta}\,
\frac{\sum_{k=1}^n \alpha_k^2}{S_n}}_{\mathbf{M}_2}.
\end{aligned}
\label{eq:avg_grad_bound_recall}
\end{equation}

We treat the three terms on the right-hand side separately.

\medskip\noindent\textbf{(i) The initial term.}
The first term is simply
\[
\frac{\mathbb{E}[f(x_1)]-\inf f}{\eta S_n}
= \frac{\mathrm{I}_1}{\eta S_n}.
\]
Its behaviour is determined by the growth of \(S_n=\sum_{k=1}^n \alpha_k\).
For the standard polynomial stepsize \(\alpha_k \asymp c_\alpha k^{-a}\)
(with \(a\neq 1\)), we have \(S_n \asymp c_\alpha\,n^{1-a}\), so the first
term is \(O(n^{a-1})\). 

\medskip\noindent\textbf{(ii) The \(\mathbf{M}_1\) term.}
Consider the numerator
\[
\sum_{k=1}^n \alpha_k(1+k)^{-\gamma/2}
\asymp \sum_{k=1}^n k^{-a-\gamma/2}.
\]
Set \(p := a+\gamma/2\). By the integral test, or equivalently by standard
\(p\)-series asymptotics, we have
\[
\sum_{k=1}^n k^{-p}
=
\begin{cases}
\mathcal{O}\!\left(n^{1-p}\right), & p<1,\\[4pt]
\mathcal{O}(1), & p>1.
\end{cases}
\]
Since \(S_n \asymp n^{1-a}\), it follows that
\[
\frac{\sum_{k=1}^n \alpha_k(1+k)^{-\gamma/2}}{S_n}
=
\mathcal{O}\bigl((1+n)^{-s_1}\bigr),
\qquad
s_1:=\min\!\Bigl\{\frac{\gamma}{2},\,1-a\Bigr\}.
\]
Therefore,
\[
\frac{2|\mathcal{A}^c|}{\eta}\,
\frac{\sqrt{\widetilde C_1}\,m^{1/2}\,\sum_{k=1}^n \alpha_k(1+k)^{-\gamma/2}}{S_n}
=
D_1\,m^{1/2}(1+n)^{-s_1},
\]
where
\[
D_1=\frac{2|\mathcal{A}^c|}{\eta}\sqrt{\widetilde C_1},
\]
up to the harmless multiplicative constants coming from the standard
summation estimate.

\medskip\noindent\textbf{(iii) The \(\mathbf{M}_2\) term.}
For the third term, we examine
\[
\sum_{k=1}^n \alpha_k^2 \asymp \sum_{k=1}^n k^{-2a}.
\]
Applying the same \(p\)-series argument with \(q:=2a\), we obtain
\[
\sum_{k=1}^n k^{-2a}
=
\begin{cases}
\mathcal{O}\!\left(n^{1-2a}\right), & 2a<1,\\[4pt]
\mathcal{O}(1), & 2a>1.
\end{cases}
\]
Dividing by \(S_n \asymp n^{1-a}\) gives
\[
\frac{\sum_{k=1}^n \alpha_k^2}{S_n}
=
\mathcal{O}\bigl((1+n)^{-s_2}\bigr),
\qquad
s_2:=\min\{a,\,1-a\}.
\]
Thus,
\[
\frac{L K_5^2}{2\eta}\,
\frac{\sum_{k=1}^n \alpha_k^2}{S_n}
=
D_2\,(1+n)^{-s_2},
\]
where
\[
D_2=\frac{L K_5^2 \alpha}{2\eta},
\]
again up to the harmless multiplicative constants coming from the exact
stepsize prefactor.

\medskip\noindent\textbf{Conclusion.}
Combining the three estimates above and substituting them into
\eqref{eq:avg_grad_bound_recall}, we obtain
\[
\frac{1}{S_n}\sum_{k=1}^n \alpha_k\,\mathbb{E}\|\nabla f(x_k)\|_1
\le
\frac{\mathrm{I}_1}{\eta S_n}
+ D_1\, m^{1/2}\,(1+n)^{-s_1}
+ D_2\,(1+n)^{-s_2},
\]
where \(s_1\) and \(s_2\) are as defined above and
\[
D_1=\frac{2|\mathcal{A}^c|}{\eta}\sqrt{\widetilde C_1},
\qquad
D_2=\frac{L K_5^2 \alpha}{2\eta}.
\]
This completes the proof.
\end{proof}
Now we are in a position to prove the main Theorem \ref{thm:rate_optimal}

\subsection{Proof of Theorem \ref{thm:rate_optimal}}

\textbf{Different Noise Model}
In this case, we have $K_4=0$. Consequently,
$$
\gamma = \min\{2p,\; b-2p,\; 2a-2b\}, \qquad s = \min\Bigl\{\frac{\gamma}{2},\, a,\, 1-a\Bigr\}.
$$
Substituting $\gamma$ into $s$ and noting that $a-b < a$, we obtain the simplified bottleneck rate:
$$
s = \min\Bigl\{p,\; \frac{b-2p}{2},\; a-b,\; 1-a\Bigr\}.
$$
We maximize $s$ over the admissible range $0 < p < b/2 < a < 1$ by equalizing the active constraints:
$$
p = \frac{b-2p}{2} = a-b = 1-a =: s.
$$
Solving this system yields $s=1/6$, which uniquely determines the optimal exponents:
$$
a=\frac56, \qquad b=\frac23, \qquad p=\frac16.
$$
These choices strictly satisfy the required admissibility conditions $a>b>0$ and $p<b/2$. Setting $\alpha=\beta=1/(mN)$ and applying these optimal parameters to the decay exponents in Corollary~\ref{cor:ergodic_constants}, we obtain the final bound:
$$
\frac{1}{S_n}\sum_{k=1}^n \alpha_k\,\mathbb{E}\|\nabla f(x_k)\|_1 = \mathcal{O}\!\left(m^{3/2}N\,n^{-1/6+\epsilon}\right).
$$

\textbf{Same Noise Model}
In this special case, the function evaluations share the same noise ($K_3 = K_4 = 0$). From Table~\ref{tab:constants_tracking_error}, this implies $C_1 = 0$ and $C_4 = 0$. Consequently,
$$
\gamma = \min\{2p,\; 2a-2b\}.
$$
Substituting this into the global bottleneck rate $s = \min\bigl\{\frac{\gamma}{2},\, a,\, 1-a\bigr\}$, and noting that $a-b < a$, we obtain:
$$
s = \min\{p,\; a-b,\; 1-a\}.
$$
Choosing
\[
a=\tfrac12,\qquad b=\varepsilon>0,\qquad p>\tfrac12,
\]
yields \(\gamma=1-2\varepsilon\) and hence
\[
s=\tfrac12-\varepsilon.
\]
Setting $\alpha=\beta=1/(mN)$ and applying these optimal parameters to the decay exponents in Corollary~\ref{cor:ergodic_constants}, we recover the convergence rate:
$$
\frac{1}{S_n}\sum_{k=1}^n \alpha_k\,\mathbb{E}\|\nabla f(x_k)\|_1 = \mathcal{O}\!\left(m^{3/2}N\,n^{-(1/2-\epsilon)}\right).
$$

\subsection{Proof of Theorem \ref{thm:rate_optimalfirstorder}}

\textbf{First-Order Feedback Model}
In the first-order setting,  $K_1 = K_3 = 0$, which consequently forces $C_1 = 0$ and $C_2 = 0$ in Table \ref{tab:constants_tracking_error}. Then 
$$
\gamma = \min\{2a,\; 2b,\; 2a-b\}.
$$
Since $a > b > 0$, it strictly holds that $2a > 2a-b$, making the $2a$ term redundant. Thus, $\gamma = \min\{2b,\; 2a-b\}$. 

Note that $s = \min\bigl\{\frac{\gamma}{2},\, a,\, 1-a\bigr\}$, and by substituting $\gamma$ we obtain
$$
s = \min\Bigl\{b,\; a-\frac{b}{2},\; 1-a\Bigr\}.
$$

 choose:
$$
a = \frac{3}{4} + \epsilon, \qquad b = \frac{1}{2}.
$$
 Applying these  parameters to Corollary~\ref{cor:ergodic_constants}, we obtain
$$
\frac{1}{S_n}\sum_{k=1}^n \alpha_k\,\mathbb{E}\|\nabla f(x_k)\|_1 = \mathcal{O}\!\left(m^{3/2}N\,n^{-1/4+\epsilon}\right).
$$

\section{Details of Compute Resources.} \label{expsetup} 
The results reported in Figure~\ref{fig:wall_clock_nn_gn_alie} were obtained on a cloud instance with 16 vCPUs (3+ GHz) and 32 GB RAM, without GPU acceleration. The experiments reported in Table~\ref{tab:time_nn_24pct} were conducted on a local machine with 16 vCPUs (5 GHz) and 32 GB RAM.

\end{document}